\documentclass[english,a4paper,12pt]{article}

\usepackage{amsxtra, amsfonts, amssymb, amstext, amsmath, mathtools}
\usepackage{amsthm}
\usepackage{stmaryrd}
\usepackage{hyperref}
\usepackage{booktabs}
\usepackage{nicefrac}
\usepackage{xspace}
\usepackage{url}
\usepackage{graphics,color}
\usepackage[algo2e,ruled,vlined,linesnumbered]{algorithm2e}
\usepackage{wrapfig}
\usepackage{algorithm}
\usepackage{algorithmic}
\usepackage{dsfont}
\usepackage{comment}
\usepackage[thinlines]{easytable}

\usepackage{lmodern}

\newtheorem{theorem}{Theorem}
\newtheorem{lemma}[theorem]{Lemma}
\newtheorem{corollary}[theorem]{Corollary}
\newtheorem{definition}[theorem]{Definition}

\newcommand{\oea}{\mbox{$(1 + 1)$~EA}\xspace}
\newcommand{\ofea}{$(1 + 1)$~FEA$_{\beta}$\xspace}

\newcommand{\oplea}{\mbox{$(1+\lambda)$~EA}\xspace}

\newcommand{\mplea}{\mbox{$(\mu+\lambda)$~EA}\xspace}
\newcommand{\mclea}{\mbox{$(\mu,\lambda)$~EA}\xspace}

\newcommand{\ollga}{\mbox{$(1+(\lambda,\lambda))$~GA}\xspace}

\newcommand{\OM}{\textsc{OneMax}\xspace}
\newcommand{\onemax}{\OM}
\newcommand{\LO}{\textsc{Leading\-Ones}\xspace}
\newcommand{\leadingones}{\LO}

\newcommand{\jump}{\textsc{Jump}\xspace}

\DeclareMathOperator{\Sample}{Sample}

\DeclareMathOperator{\gap}{gap}
\DeclareMathOperator{\im}{Im}

\newcommand{\R}{\ensuremath{\mathbb{R}}}

\newcommand{\N}{\ensuremath{\mathbb{N}}} 

\newcommand{\eps}{\varepsilon} 

\newcommand{\new}{\textcolor{black}}

\newcommand{\merk}[1]{\textbf{\textcolor{red}{#1}}}

\newcommand{\assign}{\leftarrow}
{\sloppy
\title{An Extended Jump Functions Benchmark for the Analysis of Randomized Search Heuristics\thanks{Full version of the paper~\cite{BamburyBD21} \new{that appeared} in the proceedings of GECCO 2021}}

\author{Henry Bambury\thanks{\'Ecole Polytechnique, Institut Polytechnique de Paris, Palaiseau, France} \and Antoine Bultel\footnotemark[2] \and Benjamin Doerr\thanks{Laboratoire d'Informatique (LIX), CNRS, \'Ecole Polytechnique, Institut Polytechnique de Paris, Palaiseau, France}}

\date{\today}

\begin{document}

\maketitle

\begin{abstract}
Jump functions are the \new{most-studied} non-unimodal benchmark in the theory of randomized search heuristics, in particular, evolutionary algorithms (EAs). They have significantly improved our understanding of how EAs escape from local optima. However, their particular structure -- to leave the local optimum one can only jump directly to the global optimum -- raises the question of how representative such results are.

    For this reason, we propose an extended class $\textsc{Jump}_{k,\delta}$ of jump functions that contain a valley of low fitness of width $\delta$ starting at distance $k$ from the global optimum. We prove that several previous results extend to this more general class: for all \new{$k \le \frac{n^{1/3}}{\ln{n}}$} and $\delta < k$, the optimal mutation rate for the $(1+1)$~EA is $\frac{\delta}{n}$, and the fast $(1+1)$~EA runs faster than the classical $(1+1)$~EA by a factor super-exponential in $\delta$. However, we also observe that some known results do not generalize: the randomized local search algorithm with stagnation detection, which is faster than the fast $(1+1)$~EA by a factor polynomial in $k$ on $\textsc{Jump}_k$, is slower by a factor polynomial in $n$ on some $\textsc{Jump}_{k,\delta}$ instances.

    Computationally, the new class allows experiments with wider fitness valleys, especially when they lie further away from the global optimum.
\end{abstract}

{\sloppy
\section{Introduction}

The theory of randomized search heuristics, which is predominantly the theory of evolutionary algorithms, has made tremendous progress in the last thirty years. Starting with innocent-looking questions like how the $(1+1)$ evolutionary algorithm (\oea) optimizes the \onemax function \new{(that associates to any bitstring in $\{0,1\}^n$ the number of ones it contains)}, we are now able to analyze the performance of population-based algorithms, ant colony optimizers, and estimation-of-distribution algorithms on various combinatorial optimization problems, and this also in the presence of noisy, stochastic, or dynamically changing problem data. 

This progress was made possible by the performance analysis on simple benchmark problems such as \onemax, linear functions, monotonic functions, \leadingones, long paths functions, and jump functions, which allowed to rigorously and in isolation study how EAs cope with certain difficulties. It is safe to say that these benchmark problems are a cornerstone of the theory of EAs.

Regarding the theory of EAs so far (and we refer to Section~\ref{sec:previous} for a short account of the most relevant previous works), we note that our understanding of how unimodal functions are optimized is much more profound than our understanding of how EAs cope with local optima. This is unfortunate since it is known that getting stuck in local optima is one of the key problems in the use of EAs. This discrepancy is also visible from the set of classic benchmark problems, which contains many unimodal problems or problem classes, but much fewer multimodal ones. In fact, the vast majority of the mathematical runtime analyses of EAs on multimodal benchmark problems regard only the jump function class.

Jump functions are multimodal, but have quite a particular structure. The set of easy-to-reach local optima consists of all search points in Hamming distance $k$ from the global optimum. All search points closer to the optimum (but different from it) have a significantly worse fitness. Consequently, the only way to improve over the local optimum is to directly move to the global optimum. This particularity raises the question to what extent the existing results on jump functions generalize to other problems with local optima.

The particular structure of jump functions is also problematic from the view-point of experimental studies. Since many classic EAs need time at least $n^k$ to optimize a jump function with jump size $k$, experiments are possible only for moderate problem sizes $n$ and very small jump sizes $k$, e.g., $n \le 160$ and $k=4$ in the recent work~\cite{RajabiW20}. This makes it difficult to paint a general picture, and in particular, to estimate the influence of the jump size $k$ on the performance. 

To overcome these two shortcomings, we propose a natural extension of the jump function class. It has a second parameter $\delta$ allowing the valley of low fitness to have any width lower than $k$ (and not necessarily exactly~$k$). Hence the function $\jump_{k,\delta}$ agrees with the \onemax function except that the fitness is very low for all search points in Hamming distance $k-\delta+1, \dots, k-1$ from the optimum, creating a gap of size $\delta$ (hence $\jump_{k,k}$ is the classic jump function $\jump_k$). 

Since we cannot discuss how all previous works on jump functions extend to this new benchmark, we concentrate on the performance of the \oea with fixed mutation rate and two recently proposed variations, the \ofea and the random local search with robust stagnation detection algorithm~\cite{RajabiW21evocop}. Both were developed using insights from the classic \jump benchmark. 

\emph{Particular results:} For all \new{$k \le \frac{n^{1/3}}{\ln{n}}$} and all $\delta \in [2..k]$, we give a precise estimate (including the leading constant) of the runtime of the \oea for a broad range of fixed mutation rates $p$, see Lemma~\ref{gb3}. With some more arguments, this allows us to show that the unique asymptotically optimal mutation rate is $\delta/n$, and that already a small constant-factor deviation from this value (in either direction) leads to a runtime increase by a factor of $\exp(\Omega(\delta))$, see Theorem~\ref{optEA}. The runtime obtained with this optimal mutation rate is lower than the one stemming from the standard mutation rate $p=1/n$ by a factor of $\Omega((\delta/e)^{\delta})$. These runtime estimates also allow to prove that the fast \oea with power-law exponent $\beta$ yields the same runtime as the \oea with the optimal fixed mutation rate apart from an $O(\delta^{\beta-0.5})$ factor (Theorem~\ref{ht new upper bound}), which appears small compared to the gain over the standard mutation rate. These results perfectly extend the previous knowledge on classic jump functions.

We also conduct a runtime analysis of the random local search with robust stagnation detection algorithm (SD-RLS$^*$). We determine its runtime precisely apart from lower order terms for all $\delta > 2$ and $k\le n-\omega(\sqrt{n})$  (Theorem~\ref{runtime_sd_rls}). In particular, we show that the SD-RLS$^*$, which is faster than the fast \oea by a factor polynomial in $k$ on $\jump_k$, is slower by a factor polynomial in $n$ on some $\jump_{k,\delta}$ instances. All runtime results are summarized in Table~\ref{tab:rts}.

\begin{table}[!ht] 
\resizebox{\textwidth}{!} {%
\begin{TAB}(r,5pt,10pt)[5pt]{|c|c|c|}{|c|c|c|c|}
     \textbf{Algorithm} & \textbf{$\jump_{k}$} & \textbf{$\jump_{k,\delta}$ with \new{$k \le \frac{n^{1/3}}{\ln{n}}$}} \\
      \oea with optimal MR & $\Theta((\frac{k}{n})^{-k}(\frac{n}{n-k})^{n-k})$~\cite{DoerrLMN17} & $(1+o(1))(\frac{en}{\delta})^{\delta}\binom{k}{\delta}^{-1}$\new{ [Theorem~\ref{optEA}]} \\
    \ofea  & $O(C^{\beta}_{n/2}k^{\beta - 0.5} (\frac{k}{n})^{-k}(\frac{n}{n-k})^{n-k})$~\cite{DoerrLMN17}  & $O(C^{\beta}_{n/2}\delta^{\beta - 0.5} (\frac{en}{\delta})^{\delta}\binom{k}{\delta}^{-1})$ \new{ [Theorem~\ref{ht new upper bound}]}\\
    SD-RLS$^*$ & $\binom{n}{k}(1+O(\frac{k^2}{n-2k}\ln(n)))$~\cite{RajabiW21evocop} & $(1+o(1))\big(\ln(R)\sum_{i=1}^{\delta-1}\new{\sum_{j=0}^i}\binom{n}{j}+\binom{n}{\delta}\binom{k}{\delta}^{-1}\big)$ \new{ [Theorem~\ref{runtime_sd_rls}]}
\end{TAB}
}
\caption{Summary of the runtimes of the algorithms studied in this paper on classic and our jump functions}\label{tab:rts}
\end{table}

Our experimental work in Section~\ref{experiments} shows that these asymptotic runtime differences are visible already for moderate problem sizes.

Overall, we believe that these results demonstrate that the larger class of jump functions proposed in this work has the potential to give new and relevant insights on how EAs cope with local optima.

\section{State of the Art}\label{sec:previous}

To put our work into context, we now briefly describe the state of the art and the most relevant previous works. Started more than thirty years ago, the theory of evolutionary computation, in particular, the field of runtime analysis, has first strongly focused on unimodal problems. Regarding such easy problems is natural when starting a new field and despite the supposed ease of the problems, many deep and useful results have been obtained and many powerful analysis methods were developed. We refer to the textbooks~\cite{NeumannW10,AugerD11,Jansen13,DoerrN20} for more details.

While the field has not exclusively regarded unimodal problems, it cannot be overlooked that only a small minority of the results discuss problems with (non-trivial) local optima. Consequently, not too many multimodal benchmark problems have been proposed. Besides sporadic results on cliff, hurdle, trap, and valley functions or the TwoMax and DeceivingLeadingBlocks problems (see, e.g.,~\cite{Prugel04,JagerskupperS07,FriedrichOSW09,PaixaoHST17,LissovoiOW19,OlivetoPHST18,LehreN19foga,NguyenS20,CovantesS20,WangZD21,DoerrK21ecj}) or custom-tailored example functions designed to demonstrate a particular effect, the only widely used multimodal benchmark is the class of jump functions.

Jump functions were introduced in the seminal work~\cite{DrosteJW02}. The jump function with jump parameter (jump size) $k$ is the pseudo-Boolean function that agrees with the \onemax function except that the fitness of all search points in Hamming distance $1$ to $k-1$ from the optimum is low and deceiving, that is, increasing with increasing distance from the optimum. Consequently, it comes as no surprise that simple elitist mutation-based EAs suffer from this valley of low fitness. They easily reach the local optimum (consisting of all search points in Hamming distance exactly $k$ from the optimum), but then have no other way to make progress than by directly generating the global optimum. When using standard bit mutation with the classic mutation rate, this takes an expected time of $n^k (1-1/n)^{-(n-k)} \ge n^k$. Consequently, as proven in~\cite{DrosteJW02}, the expected runtime of the \oea is $\Theta(n^k)$ when $k \ge 2$ (for $k=1$, the jump function equals \onemax and thus is unimodal). For reasonable ranges of the parameters, this bound can easily be extended to the \mplea~\cite{Doerr20gecco}. Interestingly, as also shown in~\cite{Doerr20gecco}, comma selection does not help here: for large ranges of the parameters, the runtime of the \mclea is the same (apart from lower order terms) as the one of the \mplea. This result improves over the much earlier $\exp(\Omega(k))$ lower bound of~\cite[Theorem~5]{Lehre10}.

For the \oea, larger mutation rates can give a significant speed-up, which is by a factor of order $(1-o(1))(k/e)^{k}$ for $k = o(\sqrt n)$ and the asymptotically optimal mutation rate $p = k/n$. A heavy-tailed mutation operator using a random mutation rate sampled from a power-law distribution with exponent $\beta>1$ (see~\cite{DoerrDK18,DoerrDK19} for earlier uses of random mutation rates) obtains a slightly smaller speed-up of $\Omega(k^{-\beta+0.5} (k/e)^k)$, but does so without having to know the jump size $k$~\cite{DoerrLMN17}. In~\cite{RajabiW20,RajabiW21evocop,RajabiW21gecco,DoerrR22}, stagnation-detection mechanisms were investigated which obtain a speed-up of $O((k/e)^k)$, hence saving the $k^{-\beta+0.5}$ factor loss of~\cite{DoerrLMN17}, also without having to know the jump size $k$. These works are good examples showing how a mathematical runtime analysis can help to improve existing algorithms. We note that the idea to choose parameters randomly from a heavy-tailed distribution has found a decent number of applications in discrete evolutionary optimization, e.g.,~\cite{MironovichB17,FriedrichQW18,FriedrichGQW18,QuinzanGWF21,WuQT18,AntipovBD20gecco,AntipovBD20ppsn,AntipovD20ppsn,AntipovBD21gecco,DoerrZ21aaai}.

Jump functions are also the first example where the usefulness of crossover could be proven, much earlier than for combinatorial problems~\cite{FischerW04,Sudholt05,LehreY11,DoerrHK12,DoerrJKNT13} or the OneMax benchmark~\cite{DoerrJKLWW11,DoerrDE15,Sudholt17,CorusO18tec,CorusO20}. The first such work~\cite{JansenW02}, among other results, showed that a simple $(\mu+1)$ genetic algorithm using uniform crossover with rate $p_c = O(\frac{1}{kn})$ has an $O(\mu n^2 k^3 + 2^{2k} p_c^{-1})$ runtime when the population size is at least $\mu = \Omega(k \log n)$. A shortcoming of this result, noted by the authors already, is that it only applies to uncommonly small crossover rates. Exploiting also a positive effect of the mutation operation applied to the crossover result, a runtime of $O(n^{k-1} \log n)$ was shown for natural algorithm parameters by Dang et al.~\cite[Theorem~2]{DangFKKLOSS18}. For $k \ge 3$, the logarithmic factor in the runtime can be removed by using a higher mutation rate. With additional diversity mechanisms, the runtime can be further lowered to $O(n \log n + 4^k)$, see~\cite{DangFKKLOSS16}. The \ollga with optimal parameters optimizes $\jump_k$ in time $O(n^{(k+1)/2} k^{-\Omega(k)})$~\cite{AntipovDK20}, similar runtimes result from heavy-tailed parameter choices~\cite{AntipovD20ppsn,AntipovBD21gecco}.

With a three-parent majority vote crossover, among other results, a runtime of $O(n \log n)$ could be obtained via a suitable island model for all $k = O(n^{1/2 - \eps})$~\cite{FriedrichKKNNS16}. A different voting algorithm also giving an $O(n \log n)$ runtime was proposed in~\cite{RoweA19}. Via a hybrid genetic algorithm using as variation operators only local search and a deterministic voting crossover, an $O(n)$ runtime was shown in~\cite{WhitleyVHM18}. 

Outside the regime of classic evolutionary algorithms, the compact genetic algorithm, a simple estimation-of-distribution algorithm, has a runtime  of $O(n \log n + 2^{O(k)})$~\cite{HasenohrlS18,Doerr21cgajump}. The $2$-MMAS$_{ib}$ ant colony optimizer was recently shown to also have a runtime of $O(n \log n)$ when $k \le C \ln(n)$ for a sufficiently small constant~$C$~\cite{BenbakiBD21}. Runtimes of $O\left(n \binom{n}{k}\right)$ and $O\left(k \log(n) \binom{n}{k}\right)$ were given for the $(1+1)$~IA$^{\mathrm{hyp}}$ and the $(1+1)$ Fast-IA artificial immune systems, respectively~\cite{CorusOY17,CorusOY18fast}. In~\cite{LissovoiOW19}, the runtime of a hyper-heuristic switching between elitist and non-elitist selection was studied. The lower bound of order $\Omega(n \log n) + \exp(\Omega(k))$ and the upper bound of order $O(n^{2k-1}/k)$, however, are too far apart to indicate an advantage or a disadvantage over most classic algorithms. In that work, it is further stated that the Metropolis algorithm (using the 1-bit neighborhood) has an $\exp(\Omega(n))$ runtime on jump functions. 

Finally, we note that two variants of jump functions have been proposed, namely one where the global optimum can be any point $x^*$ with $\|x\|_1 > n-k$~\cite{Jansen15} and a multi-objective variant~\cite{DoerrZ21aaai}.

\section{Preliminaries}
\label{prelim}

\subsection{Definition of the $\jump_{k,\delta}$ Function}

The $\jump_k$ function, introduced by Droste, Jansen and Wegener in \cite{DrosteJW02}, is defined as 
\[
\jump_k(x) = \begin{cases} \|x\|_1 &\mbox{if } \|x\|_1\in[0..n-k]\cup\{n\},\\
- \|x\|_1 & \mbox{otherwise,}
\end{cases} 
\] 
where $\|x\|_1=\sum_{i=1}^n x_i$ is the number of $1$-bits in $x\in\{0,1\}^n$. See the graph in Figure~\ref{mn_profile} for an example. We note that the original definition in~\cite{DrosteJW02} uses different fitness values, but the same relative ranking of the search points. Consequently, all algorithms ignoring absolute fitness values (in particular, all algorithms discussed in this work) behave exactly identical on the two variants. Our definition has the small additional beauty that the jump functions agree with the \onemax function on the easy parts of the search space. 

\begin{figure}[!ht] 
\centering
\includegraphics[width=10cm]{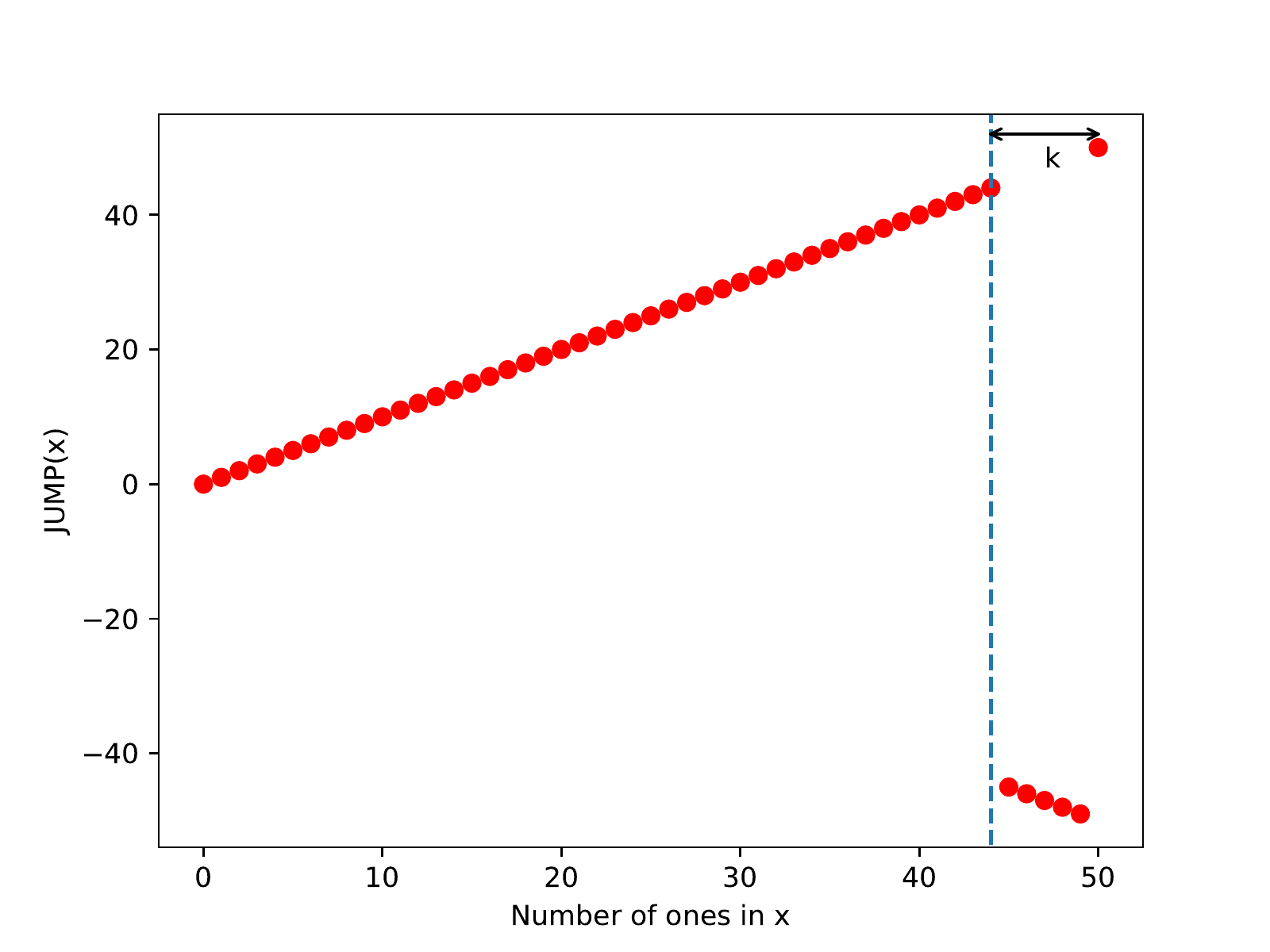}
\caption{Profile of the $\jump_6$ function.}
\label{mn_profile}
\end{figure}

The $\jump_k$ function presents several local optima (all points of fitness $n-k$). Therefore, the $\jump_k$ function allows one to analyze the ability of a given algorithm to leave a local optimum. 

However, when stuck \new{in} a local optimum, the only way to leave it is to flip all the bad bits at once, in a very unlikely \emph{perfect jump} that lands exactly on the global optimum. We speculate that this does not represent real-life problems on which evolutionary algorithms are to be applied. Indeed, in such problems local optima may exist, but usually they do not require perfection to be left.

To remedy this flaw, we propose a generalization of $\jump_k$, by defining for $\delta \in [1..k]$ the $\jump_{k,\delta}$ function via
\[
\jump_{k,\delta}(x) = \begin{cases} \|x\|_1 &\mbox{if } \|x\|_1\in[0..n-k]\cup[n-k+\delta..n],\\
- \|x\|_1 & \mbox{otherwise}
\end{cases} 
\]
for all $x\in \{0,1\}^n$.
\begin{figure}[!ht] 
\centering
\includegraphics[width=10cm]{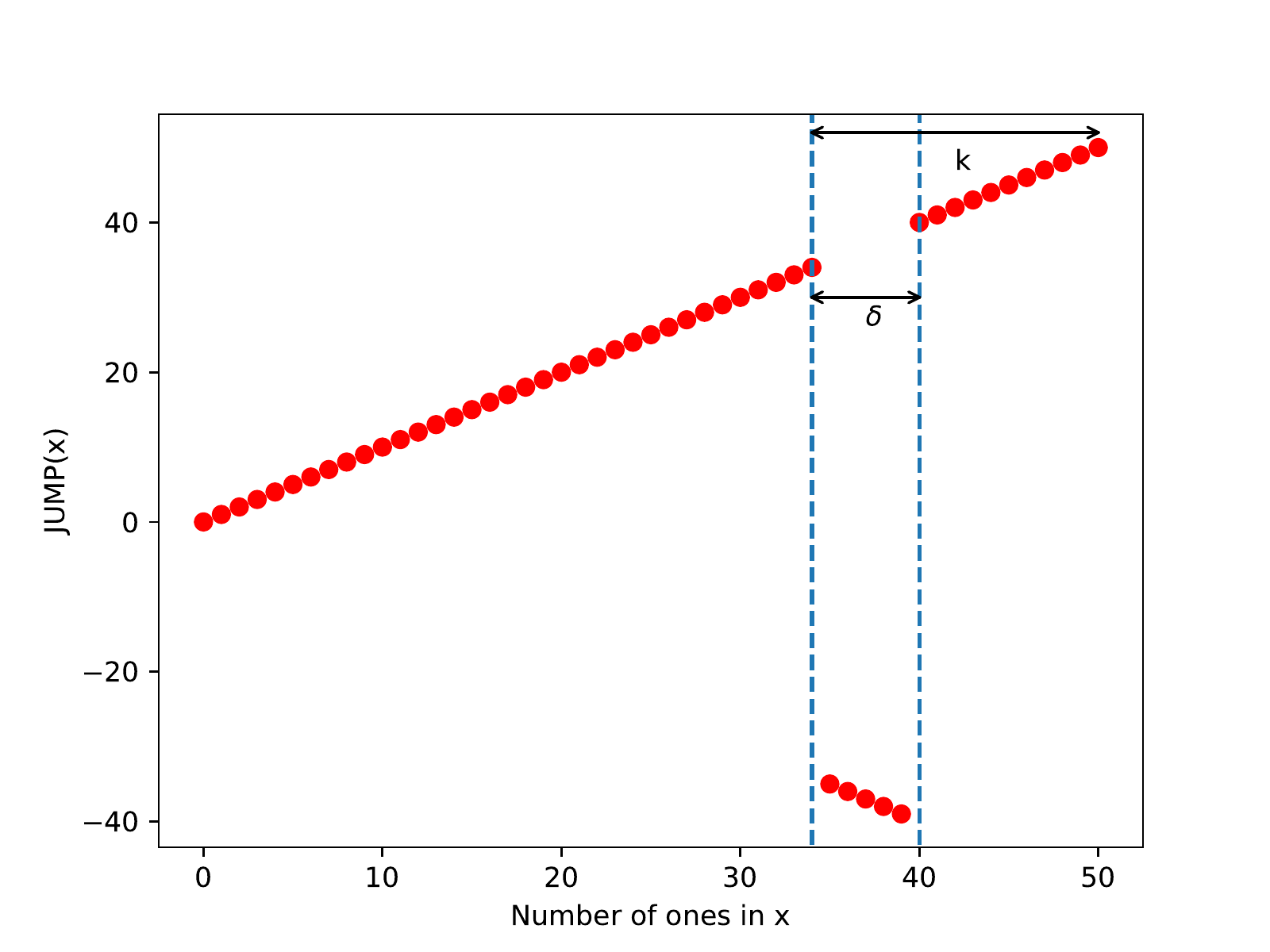}
\caption{Profile of the $\jump_{16,6}$ function.}
\label{kln_profile}
\end{figure}

The local optimum is still at Hamming distance $k$ from the global optimum, but the gap now has an arbitrary width $\delta\le k$. In particular, we observe the specific cases $\jump_{k,k}=\jump_k$ and $\new{\jump_{1,1}}=\onemax$. We also introduce the parameter $\ell:=k-\delta$. \new{This way, $n-\ell$ is the fitness of the closest individual to the local optimum that has better fitness}. \new{We note that in the classification of the block structure of a function of unitation from~\cite{LehreO18}, the function $\jump_{k,\delta}$ consists of a linear block of length~$n-k+1$, a block that for elitist algorithms starting below it is equivalent to a gap block of length~$\delta$, and another linear block of length~$k-\delta$.}

The main difference with $\jump_k$ is that with $\jump_{k,\delta}$, there are significantly more ways to jump over the valley from the local optima. This necessarily has consequences for the performance of evolutionary algorithms, as the most time-consuming phase of the optimization (jumping over the fitness valley) is now significantly easier. More precisely, from a local optimum the closest improving fitness layer contains not only $1$, but $\binom{k}{\delta}$ points with better fitness. Therefore, we intuitively expect evolutionary algorithms to be faster on $\jump_{k,\delta}$ by a factor of $\binom{k}{\delta}$. In the following sections, we will consider evolutionary algorithms whose performance of $\jump_k$ is known, and study their runtime on $\jump_{k,\delta}$, to see if they do benefit from this intuitive speedup. We note, however, that also often points above the fitness valley can be used to cross the valley, which hindered us from conducting a precise analysis for all parameter values.

\subsection{Stochastic Domination}

As we will see later on, the proposed generalization of the $\jump$ functions significantly complexifies the calculations. To ease reading, \new{whenever possible, we will rely on} the notion of stochastic domination to avoid \new{unnecessarily complicated  proofs}. This concept, introduced in probability theory, has proven very useful in the study of evolutionary algorithms~\cite{Doerr19tcs}. We gather in this subsection some useful results that will simplify the upcoming proofs. 

\begin{definition}
    Let $X$ and $Y$ be two real random variables (not necessarily defined on the same probability space). We say that $Y$ stochastically dominates $X$, denoted as $X\preceq Y$, if for all $\lambda\in\R$, $\Pr[X\geq\lambda]\leq\Pr[Y\geq\lambda]$.
\end{definition}

An elementary property of stochastic domination is the following.

\begin{lemma}\label{stochasticDomination}
    The following two conditions are equivalent.
    \begin{enumerate}
        \item $X\preceq Y$.
        \item For all monotonically non-decreasing function $f:\R\to\R$, we have $E[f(X)]\leq E[f(Y)]$.
    \end{enumerate}
\end{lemma}

Stochastic domination allows to phrase and formulate the following useful version of the statement that better parent individuals have better offspring~\cite[Lemma~6.13]{Witt13}.

\begin{lemma}\label{witt}
    Let $x,x'\in\{0,1\}^n$ such that $\|x'\|_1<\|x\|_1$, and $p\in[0,\frac{1}{2}]$. Let $y$ (resp. $y'$) be the random point in $\{0,1\}^n$ obtained by flipping each bit of $x$ (resp. $x'$) with probability $p$.
    Then, $\|y'\|_1\preceq\|y\|_1$.
\end{lemma}

\section{The \oea with Fixed Mutation Rate}
\label{fixed}
The so-called \oea is one of the simplest evolutionary algorithms. We recall its pseudocode in Algorithm \ref{alg:sdea}. The algorithm starts with a random individual $x\in\{0,1\}^n$, and generates an offspring $x'$ from $x$ by flipping each bit with probability $p$. The parameter $p$ is called the \textit{mutation rate}, and is fixed by the operator. If the offspring is not worse, i.e., $f(x')\ge f(x)$, the parent is replaced. If not, the offspring is discarded. The operation is repeated as long as desired.

\begin{algorithm2e}%
    \textbf{Initialization}\;
	$x\in \{0,1\}^n \assign \text{uniform at random}$\;
	\textbf{Optimization}\;
	\Repeat{Stopping condition}{
    $\Sample y\in\{0,1\}^n$ by flipping each bit in $x$ with probability $p$\;
    \If{$f(y)\ge f(x)$}{$ x \assign y$}
  }
\caption{\oea with static mutation rate $p$ maximizing a fitness function $f : \{0,1\}^n \to \R$}
\label{alg:sdea}
\end{algorithm2e}

A natural question when studying the \oea on a given fitness function is the determination of the optimal mutation rate. The asymptotically optimal mutation for the \oea on \onemax was shown to be~$\frac 1n$~\cite{GarnierKS99}. This result was extended to all pseudo-Boolean linear functions in~\cite{Witt13} and to the \oplea with $\lambda \le \ln n$~\cite{GiessenW17}. The optimal mutation rate of the \oea optimizing \leadingones is approximately $\frac{1.59}{n}$~\cite{BottcherDN10}, hence slightly larger than the often recommended choice~$\frac 1n$. In contrast to these results for unimodal functions, the optimal mutation rate for $\jump_k$ was shown to be $p=\frac{k}{n}$ (apart from lower-order terms); further, any deviation from this value results in an exponential (in $k$) increase of the runtime~\cite{DoerrLMN17}. It is not immediately obvious whether this generalizes to $\jump_{k,\delta}$; in this section we prove that it does under reasonable assumptions on $k,\delta,n$.

\subsection{General Upper and Lower Bounds on the Expected Runtime}

We denote by $T_p(k,\delta,n)$ the expected number of iterations of the algorithm until it evaluates the optimum. We first obtain general bounds on this expected value.
In the $\jump_{k,\delta}$ problem, as in the particular $\jump_k$ problem, the key difficulty is to leave the local optimum. To do so, the algorithm has to cross the fitness valley in one mutation step by flipping at least $\delta$ bits. The probability of this event will be crucial in our study.

\begin{definition}
    Let $i,j\in[0..n]$. We define $F_{i,j}(p)$ as the probability that, considering an individual $x$ satisfying $\|x\|_1=i$, its offspring $x'$ derived by standard bit mutation with mutation rate $p$ satisfies $\|x'\|_1\geq j$. 
    
    \new{For all $k,\delta,n \in \N$ such that $0<\delta \le k\leq n$, we will let $F(p)$ denote $F_{n-k,n-k+\delta}(p)$ to ease the reading.}
\end{definition}
The following is well-known and has been used numerous times in the theory of evolutionary algorithms. For reasons of completeness, we still decided to state the result and its proof.

\begin{lemma}\label{f(p)}
\new{For all $k,\delta,n \in \N$ such that $0<\delta \le k\leq n$, and denoting $\ell=k-\delta$, we have}
    \[
    F(p)=\sum_{j=0}^{\ell}\sum_{i=0}^{n-k}\binom{k}{\delta+i+j}\binom{n-k}{i}p^{\delta+2i+j}(1-p)^{n-\delta-2i-j}.
    \]
    Moreover, if $p\le \frac{1}{2}$, then for any $i\leq n-k$, $F_{i,n-\ell}(p)\leq F(p)$. 
\end{lemma}

\begin{proof}
Let $A_i := \{x\in\{0,1\}^n : \|x\|_1 = i\}$ for all $i$ in $[0..n]$. Consider an iteration starting with $x\in A_{n-k}$.
Let $y\in\{0,1\}^n$ be the offspring generated from $x$.
For all $j\in[0..k-\delta]$, we compute
\begin{equation*}
\begin{split}
\Pr[y&\in A_{n-\ell+j}]=\sum_{i=0}^{n-k}\Pr[\emph{we flipped $i$ 1 bits and $\delta+i+j$ 0 bits of $x$}] \\
 & = \sum_{i=0}^{n-k}\binom{k}{\delta+i+j}p^{\delta+i+j}(1-p)^{k-\delta-i-j}\binom{n-k}{i}p^i(1-p)^{n-k-i}.
\end{split}
\end{equation*}
Since the sets $(A_i)_{i\in[n-\ell..n]}$ are disjoint, we have
\begin{equation*}
\begin{split}
F(p) & =\Pr[y\in A_{n-\ell}\cup...\cup A_n]\\
 & =\sum_{j=0}^{\ell}\Pr[y\in A_{n-\ell+j}] \\
 & = \sum_{j=0}^{\ell}\sum_{i=0}^{n-k}\binom{k}{\delta+i+j}\binom{n-k}{i}p^{\delta+2i+j}(1-p)^{n-\delta-2i-j}.
\end{split}
\end{equation*}
This proves the first part of the lemma. To prove the second part, we rely on Lemma~\ref{witt}. Let $x'$ be a point of fitness $i$ for some $i\le n-k$. Let $y'$ be the offspring generated from $x'$. Since $\|x'\|_1<\|x\|_1$, the lemma states that $\|y\|_1$ stochastically dominates $\|y'\|_1$. Therefore, by definition, \[F_{i,n-k+\delta}(p)=\Pr[\|y'\|_1\geq  n-k+\delta]\leq\Pr[\|y\|_1\geq n-k+\delta]=F(p),\] which proves our claim.
\end{proof}

We now estimate the expected number of iterations needed by the \oea to optimize the $\jump_{k,\delta}$ function. \new{The following result implies that, roughly speaking, the expected runtime is $\frac{1}{F(p)}$, that is, the expected time to leave the local optimum to a better solution. This will be made more precise in Section~\ref{sec:optrate}, where also estimates for $F(p)$ will be derived.}

\begin{theorem}\label{general bound}
    For all $k,\delta,n \in \N$ such that $0<\delta \le k\leq n$ and all $p\le\frac{1}{2}$, the expected optimization time of the \oea with fixed mutation rate $p$ on the $\jump_{k,\delta}$ problem satisfies \[\frac{1}{2^n}\sum_{i=0}^{n-k}\binom{n}{i}\frac{1}{F(p)}\le T_p(k,\delta,n) \le\frac{1}{F(p)}+2\frac{\ln(n)+1}{p(1-p)^{n-1}}.\]
\end{theorem}

\begin{proof}

      Let $A_i := \{x\in\{0,1\}^n : \|x\|_1 = i\}$ for all $i$ in $[0..n]$. We call these subsets fitness levels, but note that they are not indexed in order of increasing fitness here.
      Let us start by proving the lower bound. With probability $\frac{1}{2^n}\sum_{i=0}^{n-k}\binom{n}{i}$, the initial individual of the \oea is in $A_0\cup...\cup A_{n-k}$.
      In this case, in each iteration until a fitness level of fitness not less than $n-k+\delta$ is reached, the algorithm has a positive probability of jumping over the valley. According to Lemma~\ref{f(p)}, this probability is at most $F(p)$ in each iteration. Therefore, the waiting time $W_J$ before reaching a fitness level of fitness greater than $n-k+\delta$ stochastically dominates a geometric distribution with success rate $F(p)$. Consequently, $E[W_J]\geq\frac{1}{F(p)},$ and thus 
      \[
      \frac{1}{2^n}\sum_{i=0}^{n-k}\binom{n}{i}\frac{1}{F(p)}\le T_p(k,\delta,n).
      \]
      In order to prove the upper bound, we rely on the fitness level theorem introduced by Wegener~\cite{Wegener01}. For $i\in[0..n-1]$, let
      \[
      s_i := \begin{cases} (n-i)p(1-p)^{n-1} &\mbox{if } i\in[0..n-k-1]\cup[n-k+\delta..n-1],\\
      F(p) & \mbox{if } i=n-k,\\
      ip(1-p)^{n-1} & \mbox{if } i\in[n-k+1..n-k+\delta-1].\\
      \end{cases} 
      \]
Then $s_i$ is a lower bound for the probability that an iteration starting in a point $x\in A_i$ ends with a point of strictly higher fitness. Thus, the fitness level theorem implies
\begin{equation*}
\begin{split}
 T_p(k,\delta,n) & \le \sum_{i=0}^{n-1} \frac{1}{s_i}\\
 & = \frac{1}{F(p)} + \sum_{i=0}^{n-k-1} \frac{1}{(n-i)p(1-p)^{n-1}} \\ & \quad +
 \sum_{i=n-k+1}^{n-k+\delta-1} \frac{1}{ip(1-p)^{n-1}} +\sum_{i=n-k+\delta}^{n-1} \frac{1}{(n-i)p(1-p)^{n-1}}\\
 & \le \frac{1}{F(p)} + \frac{2}{p(1-p)^{n-1}} \left(\sum_{i=1}^{n}\frac{1}{i} \right),
\end{split}
\end{equation*}
where we used that $\sum_{i=0}^{n-k-1} \frac{1}{(n-i)} +  \sum_{i=n-k+1}^{n-k+\delta-1} \frac{1}{i} + \sum_{i=n-k+\delta}^{n-1} \frac{1}{(n-i)} \le 2\sum_{i=1}^{n}\frac{1}{i}$.
With the estimate $\sum_{i=1}^n\frac{1}{i} \le \ln(n)+1$, we obtain the upper bound
\[
T_p(k,\delta,n) \le \frac{1}{F(p)} + 2\frac{\ln(n)+1}{p(1-p)^{n-1}}.\qedhere
\]
\end{proof}
We note that our lower bound only takes into account the time to cross the fitness valley. Using a fitness level theorem for lower bounds~\cite{Sudholt13,DoerrK21gecco}, one could also reflect the time spent in the easy parts of the jump function and improve the lower bound by a term similar to the $\Theta(\frac{\log n}{p(1-p)^{n-1}})$ term in the upper bound. Since the jump functions are interesting as a benchmark mostly because the crossing the fitness valley is difficult, that is, the runtime is the time to cross the valley plus lower order terms, we omit the details. We also note that only initial individuals below the gap region contribute to the lower bound. This still gives asymptotically tight bounds as long as $k \le \frac 12 n - \omega(\sqrt n)$, which is fully sufficient for our purposes. Nevertheless, we remark that the methods of Section~6 of the arXiv version of~\cite{DoerrK21gecco} would allow to show a $(1-o(1)) \frac{1}{F(p)}$ lower bound for larger ranges of parameters.

\subsection{Optimal Mutation Rate in the Standard Regime}\label{sec:optrate}

The estimates above show that the efficiency of the \oea on the $\jump_{k,\delta}$ function is strongly connected to the value $\frac{1}{F(p)}$. With the two parameters $k$ and $\delta$ possibly depending on $n$, an asymptotically precise analysis of $F(p)$ for the full parameter space appears difficult. For this reason, in most of the paper we limit ourselves to the case where \new{$k \le \frac{n^{1/3}}{\ln{n}}$} and $\delta \le k$ is arbitrary. We call this the \emph{standard regime}. For classic EAs on $\jump_k$, constant values of $k$ are already challenging and logarithmic values already lead to super-polynomial runtimes, so this regime is reasonable. \new{Note that in the standard regime we have $k=o(n^{1/3})$. This weaker condition will be sufficient in most proofs; the stronger constraint $k \le \frac{n^{1/3}}{\ln{n}}$ will only be needed in the proof of Lemma \ref{gb3}.}

\begin{lemma}\label{gb2}
    In the standard regime, if furthermore $p = o\left(\frac{1}{\sqrt{n\ell}}\right)$ with $\ell=k-\delta$ (or $p=o\left(\frac{1}{\sqrt{n}}\right)$ in the specific case $\ell=0$) we have \[F(p)= (1+o(1)) \binom{k}{\delta}p^{\delta}(1-p)^{n-\delta}.\] 
\end{lemma}

\begin{proof}

For $i \in [0..n-k]$ and $j\in [0..\ell]$, let 
\[\eps_{i,j} := \binom{k}{\delta+i+j}\binom{n-k}{i}p^{\delta+2i+j}(1-p)^{n-\delta-2i-j}.\]
Note that this describes the probability of gaining $j+\delta$ good bits by flipping $i$ \new{good} bits and $i+j+\delta$ \new{bad} bits in a string with exactly $n-k$ good bits. Recall that, by Lemma~\ref{f(p)}, $F(p)=\sum_{j=0}^\ell\sum_{i=0}^{n-k}\eps_{i,j}.$
We observe that $\eps_{i,j}=0$ for $i+j>\ell$. By reorganizing the terms in $F(p)$, noting that $\eps_{i,j} = \eps_{0,i+j}\binom{n-k}{i}\left(\frac{p}{1-p}\right)^i$, and using the binomial theorem, we compute

 \begin{equation*}
 \begin{split}
 F(p) & = \sum_{j=0}^{\ell} \sum_{i=0}^{n-k} \eps_{i,j}  = \sum_{s=0}^{\ell} \sum_{i=0}^s \eps_{0,s} \binom{n-k}{i}\left(\frac{p}{1-p}\right)^i\\
 & \le \sum_{s=0}^{\ell}\eps_{0,s}\sum_{i=0}^s \binom{s}{i}(n-k)^i\left(\frac{p}{1-p}\right)^i\\
 & = \sum_{s=0}^{\ell} \eps_{0,s} \left(1+(n-k)\frac{p}{1-p}\right)^s,
 \end{split}
\end{equation*}
where we used that $\binom{n-k}{i}\leq\frac{(n-k)^i}{i!}\leq\binom{s}{i}(n-k)^i.$
With $\frac{\eps_{0,s}}{\eps_{0,0}}= \left(\frac{p}{1-p}\right)^s\frac{\binom{k}{\delta+s}}{\binom{k}{\delta}}$ and  $\frac{\binom{k}{\delta+s}}{\binom{k}{\delta}}=\frac{\binom{\ell}{s}}{\binom{\delta+s}{s}} \le \binom{\ell}{s}$, we further estimate

\begin{equation*}
 \begin{split}
 \frac{F(p)}{\eps_{0,0}} & \le \sum_{s=0}^{\ell} \left(\frac{p}{1-p}\right)^s\frac{\binom{k}{\delta+s}}{\binom{k}{\delta}} \left(1+(n-k)\frac{p}{1-p}\right)^s\\
 & \le \sum_{s=0}^{\ell} \binom{\ell}{s} \left(\frac{p}{1-p}+(n-k)\frac{p^2}{(1-p)^2}\right)^s\\
 & = \left(1+\frac{p}{1-p}+(n-k)\frac{p^2}{(1-p)^2}\right)^{\ell}.
 \end{split}
\end{equation*}

Together with the trivial lower bound $\eps_{0,0}\le F(p)$, we thus have

\begin{equation*}
 \eps_{0,0} \le F(p) \le \eps_{0,0} \left(1+\frac{p}{1-p}+(n-k)\frac{p^2}{(1-p)^2}\right)^{\ell}. 
\end{equation*}

In the standard regime, and supposing $p=o\left(\frac{1}{\sqrt{n\ell}}\right)$, the right hand side is $\eps_{0,0}(1+o(1))$, proving our claim.
\end{proof}
\begin{lemma}\label{gb3}
    In the standard regime, if furthermore $\delta\geq 2$ and $p = o\left(\frac{1}{\sqrt{n\ell}}\right)$ (or $p=o\left(\frac{1}{\sqrt{n}}\right)$ in the specific case $\ell=0$), we have \[T_p(k,\delta,n)= (1\pm o(1))\frac{1}{\binom{k}{\delta}p^{\delta}(1-p)^{n-\delta}}.\]
\end{lemma}

\begin{proof}
We recall from the previous section that
\[
\frac{1}{2^n}\sum_{i=0}^{n-k}\binom{n}{i}\frac{1}{F(p)}\le T_p(k,\delta,n) \le\frac{1}{F(p)}+2\frac{\ln(n)+1}{p(1-p)^{n-1}}.
\]
We first compare the two terms of the upper bound. Using Lemma~\ref{gb2}, their ratio is $2F(p)\frac{\ln(n)+1}{p(1-p)^{n-1}}=(1+o(1))2\binom{k}{\delta}p^{\delta-1}(1-p)^{1-\delta}(\ln(n)+1)$. If $k=\delta$, we can already see that this \new{is} smaller than $O(n^{-\delta/2+1/2}\ln(n))=o(1)$. In the remainder we suppose $k>\delta$. Using $\binom{k}{\delta}\leq k^{\delta}$ as well as the assumptions that $p = o\left(\frac{1}{\sqrt{n\ell}}\right)\leq\frac{1}{\sqrt{n}}$ and \new{$k \le \frac{n^{1/3}}{\ln(n)}$}, we estimate
\begin{equation*}
     \begin{split}
      F(p)p^{-1}(1-p)^{1-n} & =(1+o(1))\binom{k}{\delta}\left(\frac{p}{1-p}\right)^{\delta-1}\\
     & \leq (1+o(1))k^{\delta}p^{\delta-1}\\
     & = \new{o\left(\frac{n^{\delta/3}n^{-\delta/2+1/2}}{(\ln{n})^{\delta}}\right)}.
     \end{split}
\end{equation*}
For \new{$\delta\ge 3$}, this implies
\[
F(p)\frac{\ln(n)+1}{p(1-p)^{n-1}}=o\left(\frac{\new{(\ln n)^{1-\delta}}}{n^{\delta/2-\delta/3-1/2}}\right)=o(1),
\]
as \new{$\delta/2-\delta/3-1/2\ge 0$ and $\delta \ge 2$. It only remains to study the special case where $\delta=2$. Since $k>\delta$, in this specific case, we have $\ell = k - \delta \ge \frac{1}{3}k$. Hence $p=o\left(\frac{1}{\sqrt{nk}}\right)$. This stronger constraint yields
\begin{equation*}
     \begin{split}
      F(p)p^{-1}(1-p)^{1-n} & =(1+o(1))\binom{k}{\delta}\left(\frac{p}{1-p}\right)^{\delta-1}\\
     & = O(k^{2}p)\\
     & = o\left(\frac{k^{3/2}}{\sqrt{n}}\right)\\
     & \le o\left(\frac{1}{(\ln{n})^{3/2}}\right).
     \end{split}
\end{equation*}
So again $F(p)\frac{\ln(n)+1}{p(1-p)^{n-1}}=o(1)$. Hence in both cases, we have $T_p(k,\delta,n)\le \frac{1+o(1)}{F(p)}.$}

\new{For the lower bound, $k=o(n)$ implies $\frac{1}{2^n}\sum_{i=0}^{n-k}\binom{n}{i} = 1- o(1)$ and thus $T_p(k,\delta,n)\ge \frac{1-o(1)}{F(p)}$. Consequently,} 
\[T_p(k,\delta,n)= (1\pm o(1))\frac{1}{F(p)}= (1\pm o(1))\frac{1}{\binom{k}{\delta}p^{\delta}(1-p)^{n-\delta}}.\qedhere\]
\end{proof}
We will now use the lemma above to obtain the best mutation rate in the standard regime. For this, we shall need the following two elementary mathematical results. The first is again known and was already used in~\cite{DoerrLMN17}.

\begin{lemma}\label{elementary_max}
    Let $m\in[0,n]$ The function $p\in[0,1]\mapsto p^m(1-p)^{n-m}$ is unimodal and has a unique maximum in $\frac{m}{n}$.
\end{lemma}

\begin{proof}
    The considered function is differentiable, its derivative is \[p\mapsto (p^{m-1}(1-p)^{n-m-1})(m(1-p)\new{-}(n-m)p).\] This derivative only vanishes on $p_0=\frac{m}{n}$, is positive for smaller values of $p$, and negative for larger ones. Consequently, $p_0$ is the unique maximum.
\end{proof}
\begin{corollary}\label{F_decreasing}
    $F(p)$ is decreasing in $[\frac{k+\ell}{n},1]$, where $\ell = k-\delta$.
\end{corollary}

\begin{proof}
    Recall from Lemma~\ref{f(p)} that we have $F(p)=\sum_{i,j}\eps_{i,j}$ for \[\eps_{i,j}=\binom{k}{\delta+i+j}\binom{n-k}{i}p^{\delta+2i+j}(1-p)^{n-\delta-2i-j}.\] Applying the previous lemma, $\eps_{i,j}$ is maximal for $p=\frac{\delta+2i+j}{n}$ and decreasing for larger $p$. Since $\eps_{i,j}=0$ if $i+j>\ell$, all $\eps_{i,j}$ are constant\new{ly} zero or decreasing in $[\frac{k+\ell}{n},1]$. Hence $F(p)=\sum_{i,j}\eps_{i,j}$ is decreasing in this interval as well.
\end{proof}
We now state the main result of this section. It directly extends the corresponding result from~\cite{DoerrLMN17} for classic jump functions. It shows in particular that the natural idea of choosing the mutation rate in such a way that the average number of bits that are flipped equals the number of bits that need to be flipped to leave the local optimum, is indeed correct.

\begin{theorem}\label{optEA}
    In the standard regime, the choice of $p\in[0,\frac{1}{2}]$ that asymptotically minimizes the expected runtime of the \oea on $\jump_{k,\delta}$ is $p=\frac{\delta}{n}$.
    For $\delta \ge 2$, it gives the runtime \[T_{\delta/n}(k,\delta,n)=(1\pm o(1))\binom{k}{\delta}^{-1}\left(\frac{en}{\delta}\right)^{\delta},\] and any deviation from the optimal rate by a constant factor $(1\pm\eps)$, $\eps\in(0,1)$, leads to an increase of the runtime by a factor exponential in $\delta$. 
\end{theorem}

\begin{proof}
    If $\delta = 1$, the objective function is the \onemax function, for which $p = \frac 1n$ is known to be the asymptotically optimal mutation rate~\cite{Witt13}. So let us assume $\delta \ge 2$ in the remainder. We first notice that, in the standard regime, if $k>\delta$\new{,} $\sqrt{n\ell}\frac{\delta}{n}=\frac{\sqrt{\ell}\delta}{\sqrt{n}}\leq\frac{k^{3/2}}{n^{1/2}}$. Since $k=o(n^{1/3})$, this is $o(1)$, so $\frac{\delta}{n}=o\left(\frac{1}{\sqrt{n\ell}}\right)$. If $k=\delta$, it is obvious that $\frac{\delta}{n}=o\left(\frac{1}{\sqrt{n}}\right)$. In both cases Lemma~\ref{gb3} can be applied and yields \[T_{\delta/n}(k,\delta,n)=(1\pm o(1))\binom{k}{\delta}^{-1}\left(\frac{n}{\delta}\right)^{\delta}\left(1-\frac{\delta}{n}\right)^{\delta-n}\new{.}\] Furthermore, $\left(1-\frac{\delta}{n}\right)^{\frac{n}{\delta}-1}\ge e^{-1}$ and $\left(1-\frac{\delta}{n}\right)^{n-\delta}\le \exp\left(\delta-\frac{\delta^2}{n}\right)=(1\new{-}o(1))e^{\delta}$ since $\delta = o(n^{\frac{1}{2}})$, thus $\left(1-\frac{\delta}{n}\right)^{\delta-n} = (1\new{-}o(1))e^{\delta}$, which proves our claim about $T_{\delta/n}(k,\delta,n)$. 
    
    We now prove that this runtime is optimal among all mutation rates. Note that, in the standard regime, $\frac{k+\ell}{n}=o\left(\frac{1}{\sqrt{n\ell}}\right)$ (and $\frac{k}{n}=o\left(\frac{1}{\sqrt{n}}\right)$ in the specific case $\ell=0$).
    Thus, for $p\leq \frac{k+\ell}{n}$ Lemma~\ref{gb3} can be applied and gives $T_p(k,\delta,n)= (1+o(1))\frac{1}{\binom{k}{\delta}p^{\delta}(1-p)^{n-\delta}}$. By Lemma~\ref{elementary_max}, $p^{\delta}(1-p)^{n-\delta}$ is maximized for $p=\frac{\delta}{n}$. Therefore, $T_{\delta/n}(k,\delta,n)\leq (1+o(1)) T_p(k,\delta,n)$.
    
    \new{It only remains to regard the case} $p\geq\frac{k+\ell}{n}$. We saw in Lemma~\ref{F_decreasing} that $F(p)$ is decreasing in $[\frac{k+\ell}{n},1]$. Therefore, using the lower bound from Theorem~\ref{general bound}, and applying results from the last paragraph to $\frac{k+\ell}{n}$, we obtain 
    \begin{equation*}
     \begin{split}
      T_{p}(k,\delta,n) & \geq(1\new{-}o(1))\frac{1}{F(p)}\\
     & \geq (1\new{-}o(1))\frac{1}{F(\frac{k+\ell}{n})}\\
     & \geq (1\new{-}o(1)) T_{(k+\ell)/n}(k,\delta,n)\\
     & \geq (1\new{-}o(1)) T_{\delta/n}(k,\delta,n).\\
     \end{split}
    \end{equation*}
    This shows that the mutation rate $\frac{\delta}{n}$ is asymptotically optimal among all mutation rates in $[0,\frac{1}{2}]$.
    
    Finally, let $\eps \in (0,1)$. By Lemma~\ref{gb3} we have
    \begin{equation*}
     \begin{split}
         \frac{T_{\frac{\delta}{n}(1\pm \eps)}(k,\delta,n)}{T_{\frac{\delta}{n}}(k,\delta,n)} & = \frac{(1+o(1))\binom{k}{\delta}^{-1}(\frac{\delta}{n}(1\pm \eps))^{-\delta}(1-\frac{\delta}{n}(1\pm \eps))^{\delta-n}}{(1+o(1))\binom{k}{\delta}^{-1}(\frac{\delta}{n})^{-\delta}(1-\frac{\delta}{n})^{\delta-n}}\\
     & = (1+o(1))(1\pm \eps)^{-\delta}\left(\frac{n-\delta(1\pm \eps)}{n-\delta}\right)^{\delta-n}\\
     & = (1+o(1))(1\pm \eps)^{-\delta}\left(1\mp\frac{\eps \delta}{n-\delta}\right)^{\delta-n}\\
     & = (1+o(1))\exp\left(-\delta\ln(1\pm\eps)+(\delta-n)\ln\left(1\mp \frac{\eps\delta}{n-\delta}\right)\right)\\
     & \ge (1+o(1))\exp\left(-\delta\ln(1\pm\eps)\pm\eps\delta\right),
     \end{split}
    \end{equation*}
    where we used that $\ln\left(1\pm x\right)\le \pm x$. Now $\ln(1-\eps)\le -\eps - \frac{\eps^2}{2}$ and $\ln(1+\eps)\le \eps - \frac{\eps^2}{2}+\frac{\eps^3}{3}$. Hence 
    \[
    T_{\frac{\delta}{n}(1- \eps)}(k,\delta,n) \ge (1+o(1)) \exp{\left(\delta\frac{\eps^2}{2}\right)}T_{\frac{\delta}{n}}(k,\delta,n)
    \]
    and
    \[
    T_{\frac{\delta}{n}(1+ \eps)}(k,\delta,n) \ge (1+o(1))\exp\left({\delta\left(\frac{\eps^2}{2}-\frac{\eps^3}{3}\right)}\right)T_{\frac{\delta}{n}}(k,\delta,n).
    \]
    So, any deviation from the optimal mutation rate by a small constant factor leads to an increase in the runtime by a factor exponential in $\delta$.
\end{proof}
\section{Heavy-tailed Mutation}
\label{heavytailed}

In this section we analyze the fast \oea, or \ofea, which was introduced in~\cite{DoerrLMN17}. Instead of fixing one mutation rate for the entire run, here the mutation rate is chosen randomly at every iteration, using the power-law distribution $D^{\beta}_{n/2}$ for some $\beta>1$. We recall the definition from~\cite{DoerrLMN17}.

\begin{definition}
    Let $\beta > 1 $ be a constant. Then the discrete power-law distribution $D^{\beta}_{n/2}$ on $[1..n/2]$ is defined as follows. If a random variable $X$ follows the distribution $D^{\beta}_{n/2}$, then $\Pr[X=\alpha]=(C^{\beta}_{n/2})^{-1}\alpha^{-\beta}$ for all $\alpha \in [1..n/2]$, where the normalization constant is $C^{\beta}_{n/2} := \sum_{i=1}^{\lfloor n/2\rfloor}i^{-\beta}$.
\end{definition}

\begin{algorithm2e}%
    \textbf{Initialization}\;
	$x\in \{0,1\}^n \assign \text{uniform at random}$\;
	\textbf{Optimization}\;
	\Repeat{Stopping condition}{
	Sample $\alpha$ randomly in $[1..n/2]$ according to $D^{\beta}_{n/2}$\;
    $\Sample y\in\{0,1\}^n$ by flipping each bit in $x$ with probability $\frac{\alpha}{n}$\;
    \If{$f(y)\ge f(x)$}{$ x \assign y$}
  }
\caption{The \ofea maximizing $f : \{0,1\}^n \to \R$.}
\label{alg:fea}
\end{algorithm2e}

Doerr et al.~\cite{DoerrLMN17} proved that for the $\jump_k$ function, the expected runtime of the \ofea was only a small polynomial (in $k$) factor above the runtime with the optimal fixed mutation rate.

\begin{theorem}[\cite{DoerrLMN17}]\label{ht upper bound}
    Let $n\in \N$ and $\beta>1$. For all $k\in[2..n/2]$ with $k > \beta -1$, the expected optimization time $T_{\beta}(k,n)$ of the \ofea on $\jump_k$ satisfies
    \[T_{\beta}(k,n) = O\left(C^{\beta}_{n/2}k^{\beta - 0.5} T_{\mathrm{opt}}(k,n)\right),\]
    where $T_{\mathrm{opt}}(k,n)$ is the expected runtime of the \oea with the optimal fixed mutation rate $p=\frac{k}{n}$.
\end{theorem}

We now show that this result extends to the $\jump_{k,\delta,n}$ problem. To do so, we rely on Lemma 3 (i) and (ii) of~\cite{DoerrLMN17}, restated in the following lemma.

\begin{lemma}[\cite{DoerrLMN17}]\label{Pbeta}
    There is a constant $K>0$ such that the following is true. Let $n\in\N$ and $\beta > 1$. Let $x\in\{0,1\}^n$ an individual and $y$ the offspring generated from $x$ in an iteration of the \ofea. For all $i\in[1..n/2]$ such that $i>\beta-1$, we have
    \[P_{i}^{\beta}:=\Pr[H(x,y)=i]\geq (C_{n/2}^{\beta})^{-1}Ki^{-\beta}.\]
    Moreover, with the same notations, \[P_{1}^{\beta}\geq K'{\new{(}C_{n/2}^{\beta}}\new{)}^{-1},\]
    for another constant $K'$ independent of $\beta$ and $n$.
\end{lemma}

We note that explicit lower bounds for $P_i^\beta$ for $i \le \sqrt n$ were given in the arXiv version of~\cite{DoerrZ21aaai}. We derive from the lemma above the following estimate.

\begin{corollary}\label{Pbeta_Precise}
There is a constant $\kappa>0$ such that the following is true. For all $n, k,\beta-1<\delta\leq \frac{n}{2}$, \[\binom{n}{\delta}\binom{k}{\delta}^{-1} \left(P^{\beta}_{\delta}\right)^{-1} \leq \kappa C^{\beta}_{n/2} \delta ^{\beta - 0.5} \frac{n^n}{\delta^{\delta} (n-\delta)^{(n-\delta)}}\sqrt{\frac{n}{n-\delta}}\binom{k}{\delta}^{-1}.\]
\end{corollary}

\begin{proof}
Using the Stirling approximation \[\sqrt{2\pi}n^{n+0.5}e^{-n}\le n! \le en^{n+0.5}e^{-n},\] we compute
        \begin{equation*}
        \begin{split}
         \binom{n}{\delta} & \le \frac{en^{n+0.5}e^{-n}}{\sqrt{2\pi}\delta^{\delta+0.5}e^{-\delta}\sqrt{2\pi}(n-\delta)^{n-\delta+0.5}e^{-n+\delta}}\\
         & \le \frac{e}{2\pi\delta^{0.5}}\sqrt{\frac{n}{n-\delta}}\frac{n^n}{\delta^{\delta} (n-\delta)^{(n-\delta)}}.
        \end{split}
        \end{equation*}
Combining this and Lemma~\ref{Pbeta} gives the claimed result.
\end{proof}
\begin{theorem}\label{ht new upper bound}
    Let $n\in \N$ and $\beta>1$. For $\delta \le \new{k \le \frac{n^{1/3}}{\ln{n}}}$ with $2\le \delta > \beta-1$, the expected optimization time $T_{\beta}(k,\delta,n)$ of the \ofea satisfies
    \[T_{\beta}(k,\delta,n) = O\left(C^{\beta}_{n/2}\delta^{\beta - 0.5} T_{\delta/n}(k,\delta,n)\right).\]
\end{theorem}

\begin{proof}
We use the same notation as in the proof of Theorem~$\ref{general bound}$ and Lemma~\ref{Pbeta}. For $i\in [0..n-1]$, let
\[
s_i := \begin{cases} \frac{n-i}{n}P^{\beta}_1 &\mbox{if } i\in[0..n-k-1]\cup[n-k+\delta..n-1],\\
\binom{k}{\delta}\binom{n}{\delta}^{-1}P^{\beta}_{\delta} & \mbox{if } i=n-k,\\
\frac{i}{n}P^{\beta}_1 & \mbox{if } i\in[n-k+1..n-k+\delta-1].\\
\end{cases} 
\]
Then $s_i$ is a lower bound for the probability that an iteration starting in a point $x\in A_i$ ends with a point of strictly higher fitness. Thus, the fitness level theorem~\cite{Wegener01} implies
\begin{equation*}
\begin{split}
 T_p(k,\delta,n) & \le \sum_{i=0}^{n-1} \frac{1}{s_i}\\
 & \le \binom{n}{\delta}\binom{k}{\delta}^{-1} \left(P^{\beta}_{\delta}\right)^{-1} + \sum_{i=0}^{n-k-1} \frac{n}{(n-i)P^{\beta}_1} \\ & \quad + \sum_{i=n-k+1}^{n-k+\delta-1} \frac{n}{iP^{\beta}_1} +\sum_{i=n-k+\delta}^{n-1} \frac{n}{(n-i)P^{\beta}_1}.
\end{split}
\end{equation*}
Recalling that $\frac{n^n}{\delta^{\delta} (n-\delta)^{(n-\delta)}}\sqrt{\frac{n}{n-\delta}}\binom{k}{\delta}^{-1}=(1+o(1))T_{\delta/n}(k,\delta,n)$ in the standard regime, \new{Corollary~\ref{Pbeta_Precise}} and Lemma~\ref{gb3} allow us to estimate the first term as $O(C^{\beta}_{n/2}\delta^{\beta - 0.5} T_{\delta/n}(k,\delta,n))$. Using \new{the second part of Lemma~\ref{Pbeta}}, and similar bounds as in the proof of Theorem~\ref{general bound}, we deduce that the three other terms add up to $O(\new{C^{\beta}_{n/2}}n\ln(n))$ which is $o(T_{\delta/n}(k,\delta,n))$ by Theorem~\ref{optEA}.
\end{proof}
We note without proof that the upper bound given in the theorem above is asymptotically tight.

\section{Stagnation Detection}

In this section, we study the algorithm \emph{Stagnation Detection Randomized Local Search}, or SD-RLS, proposed by Rajabi and Witt~\cite{RajabiW21evocop} as an improvement of their Stagnation Detection \oea~\cite{RajabiW20} (we note that the most recent variant~\cite{RajabiW21gecco} of this method was developed in parallel to this work and therefore could not be reflected here; \new{however, to the best of our understanding, this latest variant was optimized to deal with multiple local optima and thus does not promise to be superior to the SD-RLS algorithm on generalized jump functions. We note further that in~\cite{RajabiW21gecco} also generalized jump functions were defined. As only result for these, an upper bound of $O(\binom{n}{\delta})$ in our notation was shown under certain conditions -- consequently, similar to our analysis of SD-RLS below, this result also does not profit from the fact that $\binom{k}{\delta}$ improving solutions are available}). The algorithm builds on two ideas that have not been discussed in this work yet: stagnation detection and randomized local search.
Randomized local search is an alternative scheme to standard bit mutation. To produce an offspring from a given individual, instead of flipping each bit independently with a given probability, $s$ bits are chosen uniformly at random and flipped. Therefore the offspring is necessarily at Hamming distance $s$ from its parent. The parameter $s$ is usually referred to as the \emph{strength} of the mutation.
Stagnation detection is a heuristic introduced by Rajabi and Witt \cite{RajabiW20} that can be added to many evolutionary algorithms. When the algorithm has spent a given number of steps without fitness improvement, it increases its mutation strength, as a way to leave the local optimum it might be stuck in. The number of unsuccessful steps in a row needed to increase the strength can depend on the current strength. Its value should be thoughtfully designed, ideally so that the probability of missing an improvement at Hamming distance $s$ is small. In SD-RLS, this value was chosen to be $\binom{n}{s}\ln(R)$, where $s$ is the current strength and $R$ is a control parameter, fixed for the entire run by the user. \cite{RajabiW21evocop} typically use $R$ to be a small polynomial in $n$. We call \emph{step} $s$ the $\binom{n}{s}\ln(R)$ iterations in which strength $s$ is used.

\new{The main flaw of SD-RLS is that infinite runs are possible: for every point in the search space, there is only a finite number of strengths for which a fitness improvement is possible. If these improvements are unluckily missed by the algorithm at all the corresponding steps, SD-RLS would keep increasing the strength forever and never terminate. To avoid such a situation,} the same article introduces Randomized Local Search with Robust Stagnation Detection, or SD-RLS$^*$. When step~$s$ terminates without having found an improvement, instead of increasing the mutation strength to $s+1$ immediately, the algorithm goes back to all previous strengths, in decreasing order, before moving to step $s+1$. This mildly impacts the runtime, but ensures termination in expected finite time. Let \emph{phase} $s$ denote the succession of steps $s,...,1$ where step $s$ is the first interval in which strength $s$ is used. A run of the algorithm can now be seen as a succession of distinct \emph{phases} with increasing $s$ value. The pseudocode of both algorithms is recalled \new{in Algorithms \ref{alg:sdrls} and \ref{alg:sdrlstar}.}

\begin{algorithm2e}\label{alg:sdrls}
    \textbf{Initialization}\;
	$x\in \{0,1\}^n \assign$ uniform at random\;
	$s_1\assign 1$;	$u\assign0$\;
    \textbf{Optimization}\;
    \For{$t\assign 1,2,...$}{
    Create $y$ by flipping $s_t$ bits in a copy of $x$ uniformly\;        
    $u\assign u+1$\;
    \uIf{$f(y)>f(x)$}{$x\assign y$; $s_{t+1}\assign 1$; $u\assign0$\;}
    \ElseIf{$f(y)=f(x)$ \text{\normalfont and} $s_t=1$}{$x\assign y$\;}
    \uIf{$u>\binom{n}{s_t}\ln{R}$}{$s_{t+1}\assign \min\{s_t+1,n\}$; $u \assign 0$\;}
    \Else{$s_{t+1}\assign s_t$\;}
    }
\caption{SD-RLS maximizing $f:\{0,1\}^n\to\R$ with parameter~$R$.}
\end{algorithm2e}

In this section, we shall focus on SD-RLS$^*$, as it ensures termination and was studied more profoundly in \cite{RajabiW21evocop}. We first collect some central properties of the algorithm, before studying its performances on $\jump_{k,\delta}$. For $x\in\{0,1\}^n$, let $\gap(x):=\min\{H(x,y)\mid y\in\{0,1\}^n: f(y)>f(x)\}$ be the distance of $x$ to the closest strictly fitter point. 

The following lemma from~\cite{RajabiW21evocop} demonstrates that even if an improvement at Hamming distance $m$ is not found during phase $m$, the loop structure of the SD-RLS$^*$ ensures that some improvement will still be found in reasonable time.
\begin{lemma}\label{if_missed_sd_rls}
    Let $x\in\{0,1\}^n$, with $m:=\gap(x)<n/2$, be the current search point of the SD-RLS$^*$ optimizing a pseudo-Boolean function $f$, with $R\ge n^{1+\eps}|\im(f)|$ for some constant $\eps>0$. Let $W_x$ be the waiting time until a strict improvement is found. Let $E_m$ be the event that a strict improvement is not found during phase $m$ (all previous phases cannot create a strict improvement). Then,
    \[
    E\left[W_x\mid E_m\right]=o\left(\frac{R}{|\im(f)|}\binom{n}{m}\right).
    \]
\end{lemma}

\begin{algorithm2e}\label{alg:sdrlstar}
    \textbf{Initialization}\;
	$x\in \{0,1\}^n \assign$ uniform at random\;
	$r_1\assign 1$; $s_1\assign 1$; $u_1\assign0$\;
    \textbf{Optimization}\;
    \For{$t\assign 1,2,...$}{
    Create $y$ by flipping $s_t$ bits in a copy of $x$ uniformly\;        
    $u_{t+1}\assign u_t+1$\;
    \uIf{$f(y)>f(x)$}{$x\assign y$; $s_{t+1}\assign 1$; $r_{t+1}\assign 1$; $u_{t+1}\assign0$\;}
    \ElseIf{$f(y)=f(x)$ \text{\normalfont and} $r_t=1$}{$x\assign y$\;}
    \uIf{$u_{t+1}>\binom{n}{s_t}\ln{R}$}{
    \uIf{$s_t=1$}{
    \uIf{$r_t<n/2$}{$r_{t+1}\assign r_t+1$;}\Else{$r_{t+1}\assign n$\;}
    $s_{t+1}\assign r_{t+1}$\;
    }
    \Else{$r_{t+1}\assign r_t$; $s_{t+1}\assign s_t-1$;}
    $u_{t+1}\assign0$\;}
    \Else{$s_{t+1}\assign s_t$; $r_{t+1}\assign r_t$\;}
  }
\caption{SD-RLS$^*$ maximizing $f:\{0,1\}^n\to\R$ with {parameter}~$R$. The counters $\new{s_t}, \new{u_t}, r_t$ keep track of the current strength, the number of iterations spent in the current step, and the current phase.}
\end{algorithm2e}

This result also allows one to bound the runtime needed to leave a given search point if the number of neighbors with higher fitness is known. The following result was not explicitly stated in \cite{RajabiW21evocop}, but it is a rather direct corollary.
\begin{lemma}\label{rls_time_to_leave}
    Let $x\in\{0,1\}^n$, with $m:=\gap(x)<n/2$, be the current search point of the SD-RLS$^*$ optimizing a pseudo-Boolean function $f$, with $R\ge n^{1+\eps}|\im(f)|$ for some constant $\eps>0$, just after a fitness improvement was found. We assume that all points of fitness $f(x)$ have the same gap $m$. Let $W_x$ be the waiting time until a strict improvement is found.
    We further suppose that for any point of fitness $f(x)$, there are exactly $N$ points at Hamming distance $m$ of $x$ that have strictly higher fitness.
    Then
    \[
    E[W_x]\leq \ln(R)\sum_{i=1}^{m-1}\sum_{j=1}^i\binom{n}{j}+(1+o(1))\frac{\binom{n}{m}}{N}.
    \]
\end{lemma}
\begin{proof}
If $\gap(x)=m$, the first $m-1$ phases of the algorithm cannot lead to any fitness improvement. Their accumulated length is $\ln(R)\sum_{i=1}^{m-1}\sum_{j=1}^i\binom{n}{j}$.
We therefore regard now $W'_x=W_x-\ln(R)\sum_{i=1}^{m-1}\sum_{j=1}^i\binom{n}{j}$, which is the waiting time once strength $m$ is reached.

When the algorithm reaches phase $m$, strength $m$ is used for $\ln(R)\binom{n}{m}$ iterations. The probability of finding an improvement in one such iteration is $N/\binom{n}{m}$. Therefore, the probability of the event $E_m$ (that is, missing all those tries) is at most
\[
\Pr[E_m]\leq\left(1-N\binom{n}{m}^{-1}\right)^{\ln(R)\binom{n}{m}}\le \frac{1}{R^N}=o(1).
\]
To conclude, we use the law of total expectancy
\[
E[W'_x]=\Pr[E_m]E\left[W'_x\mid E_m\right]+\Pr[\overline{E}_m]E[W'_x\mid \overline{E}_m].
\]
Lemma~\ref{if_missed_sd_rls} implies that $E\left[W'_x\mid E_m\right]\le E\left[W_x\mid E_m\right]=o\left(\frac{R}{|\im(f)|}\binom{n}{m}\right)$. Using the aforementioned bounding of $\Pr[E_m]$, we estimate the first term as $o\left(\frac{\binom{n}{m}}{R^{N-1}|\im(f)|}\right)$.

Conditional on $\overline{E}_m$, $W'_x$ is distributed following a geometric law of parameter $N/\binom{n}{m}$ conditional on being at most $\ln(R)\binom{n}{m}$. Such a distribution is dominated by the standard geometric law of same parameter. Therefore, $E[W'_x\mid \overline{E}_m]\leq \frac{\binom{n}{m}}{N}$ and the second term is at most $\frac{\binom{n}{m}}{N}$. This clearly dominates $o\left(\frac{\binom{n}{m}}{R^{N-1}|\im(f)|}\right)$, which yields
\[
E[W'_x]\le(1+o(1))\frac{\binom{n}{m}}{N}.\qedhere
\]
\end{proof}
This bound on the runtime allows for an intuitive understanding of the performance of SD-RLS$^*$ on jump functions. When on a local optimum of gap~$m$, with $N$ possible improvements at distance $m$, the first $m-1$ steps are inefficient, which wastes $\ln(R)\sum_{i=1}^{m-1}\sum_{j=1}^i\binom{n}{j}$ iterations. But once strength $m$ is reached, the success probability of one iteration is $N/\binom{n}{m}$, which is better than standard bit mutation with rate $\frac{m}{n}$, which only has success probability $N(\frac{en}{m})^{-m}$. If the length of the inefficient steps is dominated by $\binom{n}{m}/N$, this is an advantageous trade, and the algorithm is very efficient. This is the case on $\jump_{k}$: the SD-RLS$^*$ turns out to be faster than all the algorithms we have studied until now. More precisely, it gives a speed-up of order $(\frac{en}{k})^k\binom{n}{k}^{-1}$. This theorem is one of the main results of~\cite{RajabiW21evocop}.
\begin{theorem}\label{sd_rls_jump_m}
    Let $n\in\mathbb{N}$. Let $T_{\mathrm{SD-RLS}^*}(k,n)$ be the expected runtime of the SD-RLS$^*$ on $\jump_{k}$. For all $k\ge 2$, if $R\ge n^{2+\eps}$ for some constant $\eps>0$, then
    \[
    T_{\mathrm{SD-RLS}^*}(k,n)=
    \begin{cases}
    \binom{n}{k}\left(1+O\right(\frac{k^2}{n-2k}\ln(n)\left)\right) & \mbox{if } k<n/2,\\
      O(2^nn\ln(n)) & \mbox{if } k\ge n/2.\\
      \end{cases} 
    \]
\end{theorem}

But this trade is not always advantageous, especially for large values of~$N$. Indeed in this case $\binom{n}{m}/N$ is small, and the  $\ln(R)\sum_{i=1}^{m-1}\sum_{j=1}^i\binom{n}{j}$ term (corresponding to the length of the wasted steps) dominates the other one. This term does not benefit from any speed-up when $N$ increases, and that is likely to slow down SD-RLS$^*$ in comparison to other algorithms. This is visible on $\jump_{k,\delta}$, where $N=\binom{k}{\delta}$: while other algorithms benefit from a factor $N$ runtime speedup, the SD-RLS$^*$ algorithm does not, and becomes consequently slower in comparison. In the following theorem, we obtain a precise asymptotic value for the runtime of SD-RLS$^*$ on $\jump_{k,\delta}$.
\begin{theorem}\label{runtime_sd_rls}
    Let $\tau=T_{\mathrm{SD-RLS}^*}(k,\delta,n)$ be the runtime of the SD-RLS$^*$ on $\jump_{k,\delta}$. Suppose that there exists a constant $\eps>0$ such that the control parameter is $R\ge n^{2+\eps}$. Then for all $k<\frac{n}{2}$ and $\delta\ge 2$, we have
    \begin{equation*}
    \begin{split}
    \tau & \geq\left(1 - \frac{1}{2^n}\sum_{i=0}^{k-1}\binom{n}{i}\right)\left[\ln(R)\sum_{i=1}^{\delta-1}\sum_{j=1}^i\binom{n}{j}+\binom{n}{\delta}\binom{k}{\delta}^{-1}\right],\\
    \tau & \leq \ln(R)\sum_{i=1}^{\delta-1}\sum_{j=1}^i\binom{n}{j}+(1+o(1))\left[\binom{n}{\delta}\binom{k}{\delta}^{-1}+n\ln(n)+n\right].
    \end{split}
    \end{equation*}
    If furthermore $\delta\ge 3$ and $k\le n-\omega(\sqrt{n})$, these bounds are tight and thus
    \[
    \tau=(1+o(1))\left[\ln(R)\sum_{i=1}^{\delta-1}\sum_{j=1}^i\binom{n}{j}+\binom{n}{\delta}\binom{k}{\delta}^{-1}\right].
    \]
\end{theorem}

\begin{proof}
For the lower bound, we recall that with probability $\frac{1}{2^n}\sum_{i=0}^{n-k}\binom{n}{i} = 1 - 2^{-n} \sum_{i=0}^{k-1} \binom{n}{i}$, the initial search point is sampled before the valley. Regardless of the initial search point (as long as it is below the fitness valley), for the global optimum to be sampled, two tasks have to be completed, in order. First, the strength needs to increase to at least $\delta$ (condition~(i)). Once strength $\delta$ is reached, a point above the valley has to be generated from a point below it (condition~(ii)). $\tau$ is necessarily larger than the sum of the times needed to complete both conditions. 

By definition of the algorithm, condition~(i) takes at least $\ln(R)\sum_{i=1}^{\delta-1}\sum_{j=1}^i\binom{n}{j}$ iterations to be completed. For condition~(ii), consider the first step where strength $\delta$ is reached. It is used for $\ln(R)\binom{n}{\delta}$ iterations. Consider one of these iterations: the fitness of the parent is some $i\in[0..n-k]$. The probability of jumping above the valley is $0$ if $i<n-k$, and $\binom{k}{\delta}/\binom{n}{\delta}$ otherwise. So the probability of jumping at any iteration of the step is at most $\binom{k}{\delta}/\binom{n}{\delta}$. Thus, the expected time needed to complete condition~(ii) is bigger than $E[Y]$ with $Y:=\min\{\kappa,X\}$, where $\kappa = \lceil \ln(R)\binom{n}{\delta} \rceil$ and $X$ is a variable following a geometric distribution with success rate $p:=\binom{k}{\delta}/\binom{n}{\delta}$. Using the elementary fact that $E[Y] = \sum_{i=1}^\infty \Pr[Y \ge i]$ for all random variables $Y$ taking values in the non-negative integers, we estimate
\begin{align*}
  E[Y] & = \sum_{i=1}^\infty \Pr[Y \ge i] = \sum_{i=1}^\kappa \Pr[Y \ge i]\\
	& = \sum_{i=1}^\kappa \Pr[X \ge i] = \sum_{i=1}^\kappa (1-p)^{i-1} \\
	&= \frac{1 - (1-p)^\kappa}{1 - (1-p)} = \tfrac 1 p (1 - (1-p)^\kappa) \\
	& = (1- o(1)) \tfrac 1p = (1-o(1)) \frac{\binom{n}{\delta}}{\binom{k}{\delta}},
\end{align*}
where we estimated $(1-p)^\kappa \le \exp(-p\kappa) \le \exp\big(-\ln(R) \binom{k}{\delta}\big) = o(1)$.

This proves
\[
\tau \geq\left(\frac{1}{2^n}\sum_{i=0}^{n-k}\binom{n}{i}\right)\left[\ln(R)\sum_{i=1}^{\delta-1}\sum_{j=1}^i\binom{n}{j}+(1+o(1))\binom{n}{\delta}\binom{k}{\delta}^{-1}\right].
    \]

For the upper bound, we consider the sequence of layers visited by the SD-RLS$^*$. At most one of them has gap $\delta$, the layer of the local optima. Any point in this layer has $\binom{k}{\delta}$ improving neighbors at distance $\delta$. According to Lemma~\ref{rls_time_to_leave}, the time needed to leave this layer is therefore at most $\ln(R)\sum_{i=1}^{\delta-1}\sum_{j=1}^i\binom{n}{j}+(1+o(1))\frac{\binom{n}{\delta}}{\binom{k}{\delta}}$.
All other visited layers have gap $1$. More precisely, any point in the fitness valley, of fitness $i\in[n-k+1..n-k+\delta-1]$, has at most $i$ neighbors of higher fitness at distance~1. Any point outside of the fitness valley with fitness $i\in[0..n-k-1]\cup[n-k+\delta..n-1]$ has $n-i$ neighbors of higher fitness at distance 1. Therefore, by Lemma~\ref{rls_time_to_leave}, the accumulated time needed to leave these layers is at most
\begin{equation*}
    \begin{split}
    (1+o(1))\left[\sum_{\substack{i=0\\i \notin[n-k..n-k+\delta-1] }}^{n-1}\frac{n}{n-i}+\sum_{i=n-k+1}^{n-k+\delta-1}\frac{n}{i}\right] &  \le(1+o(1)) \sum_{i=0}^{n-1}\frac{n}{n-i}\\
    & \le (1+o(1))(n\ln(n)+n),\\
    \end{split}
\end{equation*}
where we used the fact that (since $k\le n/2$) for all $i\in[n-k+1..n-k+\delta-1]$, we have $1/(n-i)\ge 1/i$. The last inequality is a classical estimate for the harmonic sum. This yields
\[
\tau \leq \ln(R)\sum_{i=1}^{\delta-1}\sum_{j=1}^i\binom{n}{j}+(1+o(1))\left[\binom{n}{\delta}\binom{k}{\delta}^{-1}+n\ln(n)+n\right].
\]

If $\delta\ge 3$, the double sum includes the term $\binom{n}{2}=\omega(n\ln(n))$. Furthermore, if $k\le n-\omega(\sqrt{n})$, then $\frac{1}{2^n}\sum_{i=0}^{n-k}\binom{n}{i}=(1-o(1))$ (this can be computed by applying Chernoff multiplicative bound to a binomial variable). This implies that the upper and lower bounds are asymptotically tight.
\end{proof}
 The following lemma shows that we can find in the standard regime instances on which SD-RLS$^*$ is slower than the standard \oea by a factor polynomial in $n$ of arbitrary degree. Note that this is all the more true when comparing to the optimal \oea or the \ofea.

\begin{lemma}
    Consider the instance of $\jump_{k,\delta}$ where $\delta \ge 2$ is constant and $k=n^{K/\delta}$ for some \new{constant} $K<\delta/3$. \new{For $n$ large enough, this instance is in the standard regime, and satisfies }
    \[
    \frac{T_{\mathrm{SD-RLS}^*}(k,\delta,n)}{T_{\frac{1}{n}}(k,\delta,n)}=\Omega(n^{K-1}).
    \]
\end{lemma}

\begin{proof}
\new{Since $K/\delta<1/3$, $k = n^{K/\delta}$ is below $\frac{n^{1/3}}{\ln(n)}$ when $n$ is large enough. Hence this instance is in the standard regime as defined in Section~\ref{fixed}.}

We recall from Lemma~\ref{gb3} that, in this case, the runtime of the \oea is $(1+o(1))\binom{k}{\delta}^{-1}n^{\delta}(1-\frac{1}{n})^{\delta-n}$. Since $\delta$ is constant, this can be simplified to $O(k^{-\delta}n^{\delta})$. The runtime of SD-RLS$^*$ is $(1+o(1))\left[\ln(R)\sum_{i=1}^{\delta-1}\sum_{j=1}^i\binom{n}{j}+\binom{n}{\delta}\binom{k}{\delta}^{-1}\right]$, which is larger than the last term of the double sum, $\binom{n}{\delta-1}=\Omega(n^{\delta-1})$.
This yields 
\[
\frac{T_{\mathrm{SD-RLS}^*}(k,\delta,n)}{T_{\frac{1}{n}}(k,\delta,n)}=\Omega\left(\frac{n^{\delta-1}}{k^{-\delta}n^{\delta}}\right)=\Omega\left(n^{K-1}\right). \qedhere
\]
\end{proof}
\section{Experiments}
\label{experiments}
In this final section, we present experimental results that provide a concrete perspective on our study. Since all our theoretical results are asymptotic, a natural question is to what extent they can be observed on reasonable problem instances. To answer this question, we implemented the aforementioned algorithms and executed them on test instances. We also implemented and tested another algorithm, the SD-\oea introduced in $\cite{RajabiW20}$. This algorithm is similar to SD-RLS, but standard bit mutation with mutation rate $\frac{s}{n}$ is used instead of $s$-bits flips. 

We tested the algorithms on five different regimes for $k$ and $\delta$. For each regime, we only used values of $n$ that led to reasonable computation times. Each point on the following graphs is the average over 1000 runs. Note that, for better readability, the following graphs do not display variances. They were computed during the experiments, but they were close to the theoretical variances of the underlying geometric distributions of the jump phenomenon (this is coherent with previous experiments on the matter, e.g., \cite{RajabiW21evocop}).

To achieve reasonable computation times, for all experiments we used what we call \emph{partial simulation}. The algorithms are executed until the encounter of a local optimum. There, we sampled the number of iterations needed to jump, using the theoretical distributions determined in this work. For the \oea with mutation rate $p$, it is a geometric law with parameter $F(p)$. For the \ofea, it is a geometric law of parameter $C_{n/2}^{\beta}\sum_{i=0}^{n/2}i^{-\beta}F(i/n)$. For the SD-RLS$^*$, one step with strength $s\ge \delta$ is equivalent to sampling geometric law of parameter $\binom{n}{s}^{-1} \sum_{i=0}^{\lfloor (s-\delta)/2 \rfloor} \binom{k}{s-i} \binom{n-k}{i}$ (the step is failed if the sampled value is greater than $\ln(R)\binom{n}{s}$). A similar technique is used for the SD-\oea. The fitness level of arrival after the jump is also sampled using theoretical distributions. The algorithm is then executed again until the end of the run. Note that the theoretical distributions used are exact, hence our partial simulation approach generates an exact sample for the true runtime distribution. 

The five regimes we chose to display illustrate a progressive spectrum from the classical jump function to larger instances of generalized jump functions, where $\binom{k}{\delta}$ becomes considerable. On Figure~\ref{reg0}, the problem is the classical $\jump_{4}$, as we take $\delta=k=4$.
On Figure~\ref{reg1}, $\delta$ and $k$ are still constants, but $\delta<k$.
On Figure~\ref{reg2}, $k=2\ln(n)$, and $\delta=\frac{k}{2}$. Figure~\ref{reg3}, $k=n^{0.3}$, $\delta=\frac{k}{2}$, represents the limit of the standard regime.  Figure~\ref{reg4}, $k=\frac{n}{4}$, $\delta=\frac{k}{2}$, explores an instance outside of this regime.
\begin{figure}[!htb] 
\centering
\includegraphics[width=10cm]{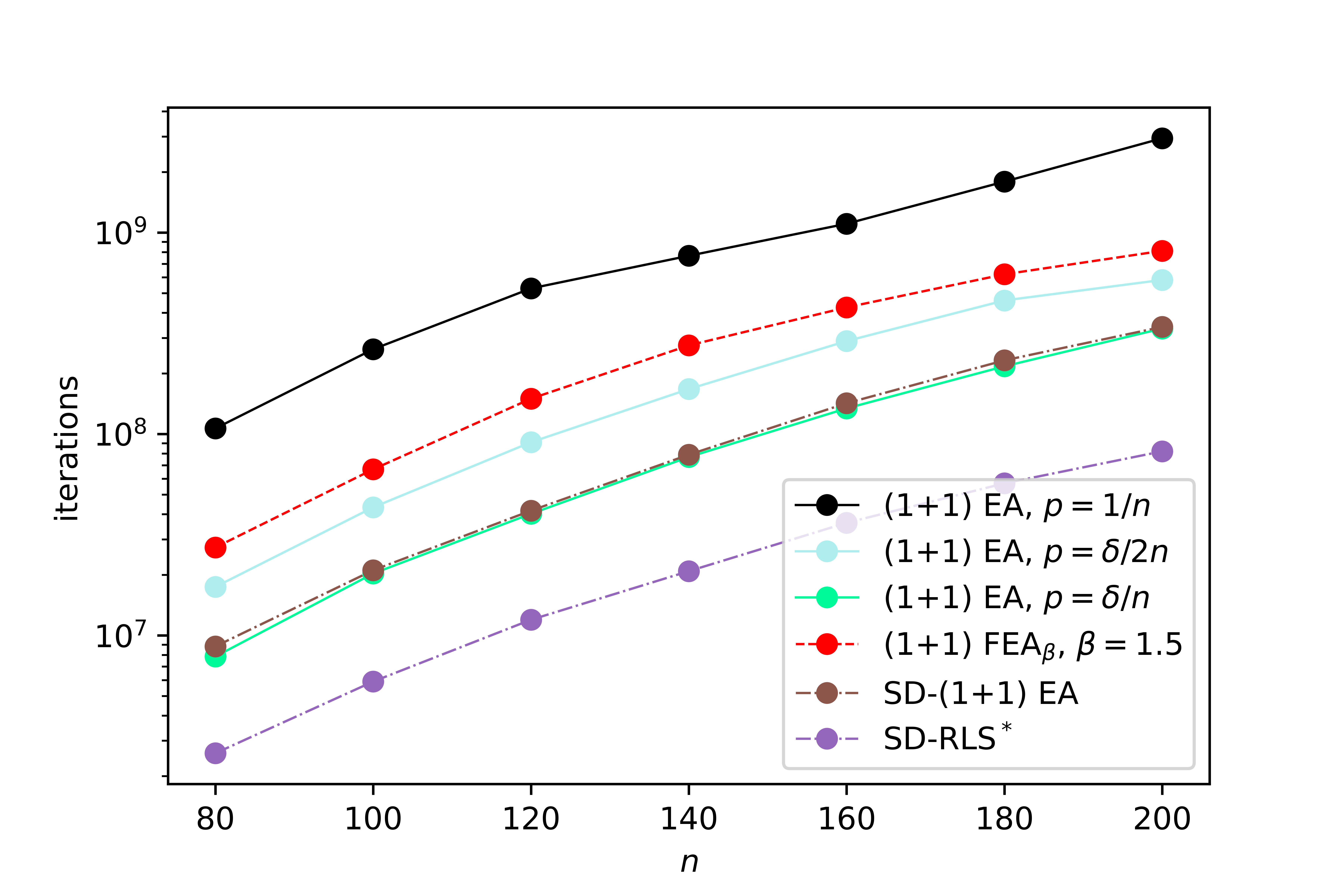}
\caption{Optimization times of different algorithms on the classic jump function $\jump_k = \jump_{k,\delta}$ with $k=\delta=4$.}
\label{reg0}
\end{figure}
\begin{figure}[!htb] 
\centering
\includegraphics[width=10cm]{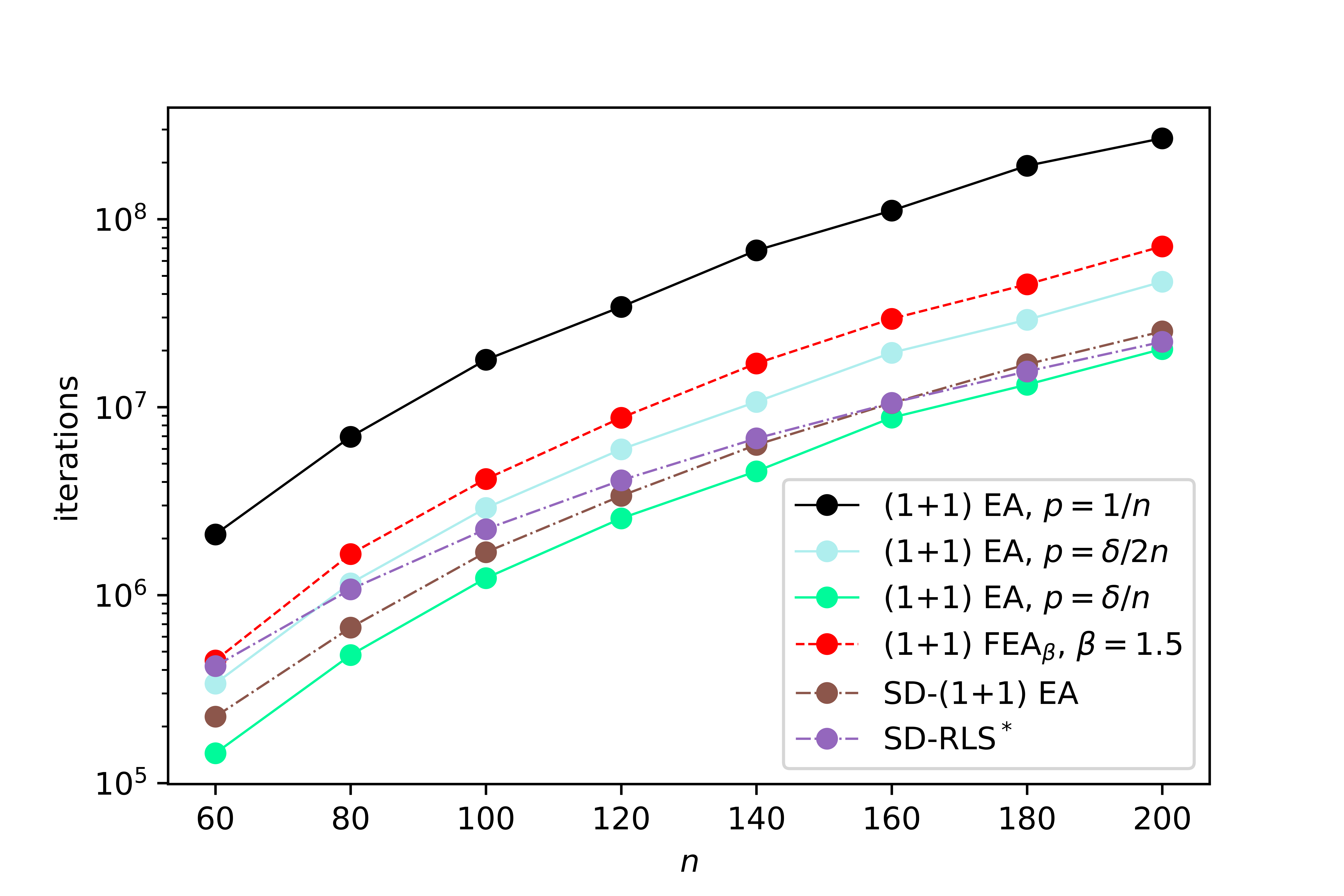}
\caption{Optimization times on $\jump_{k,\delta}$ with $k=6$, $\delta = 4$.}
\label{reg1}
\end{figure}
\begin{figure}[!htb] 
\centering
\includegraphics[width=10cm]{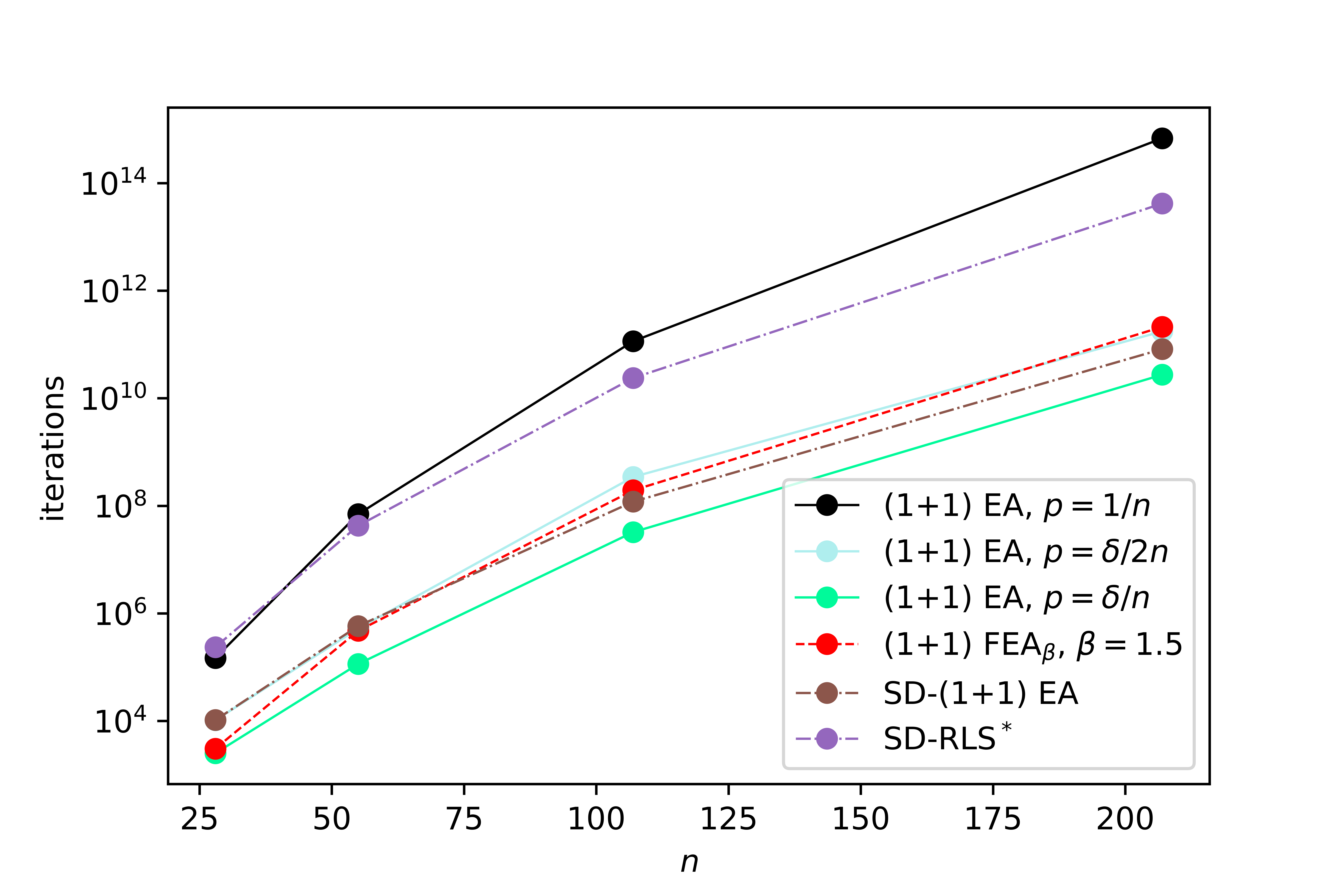}
\caption{Optimization times on $\jump_{k,\delta}$ with $k=3\ln(n)$, $\delta=\frac k2$.}
\label{reg2}
\end{figure}
\begin{figure}[!htb] 
\centering
\includegraphics[width=10cm]{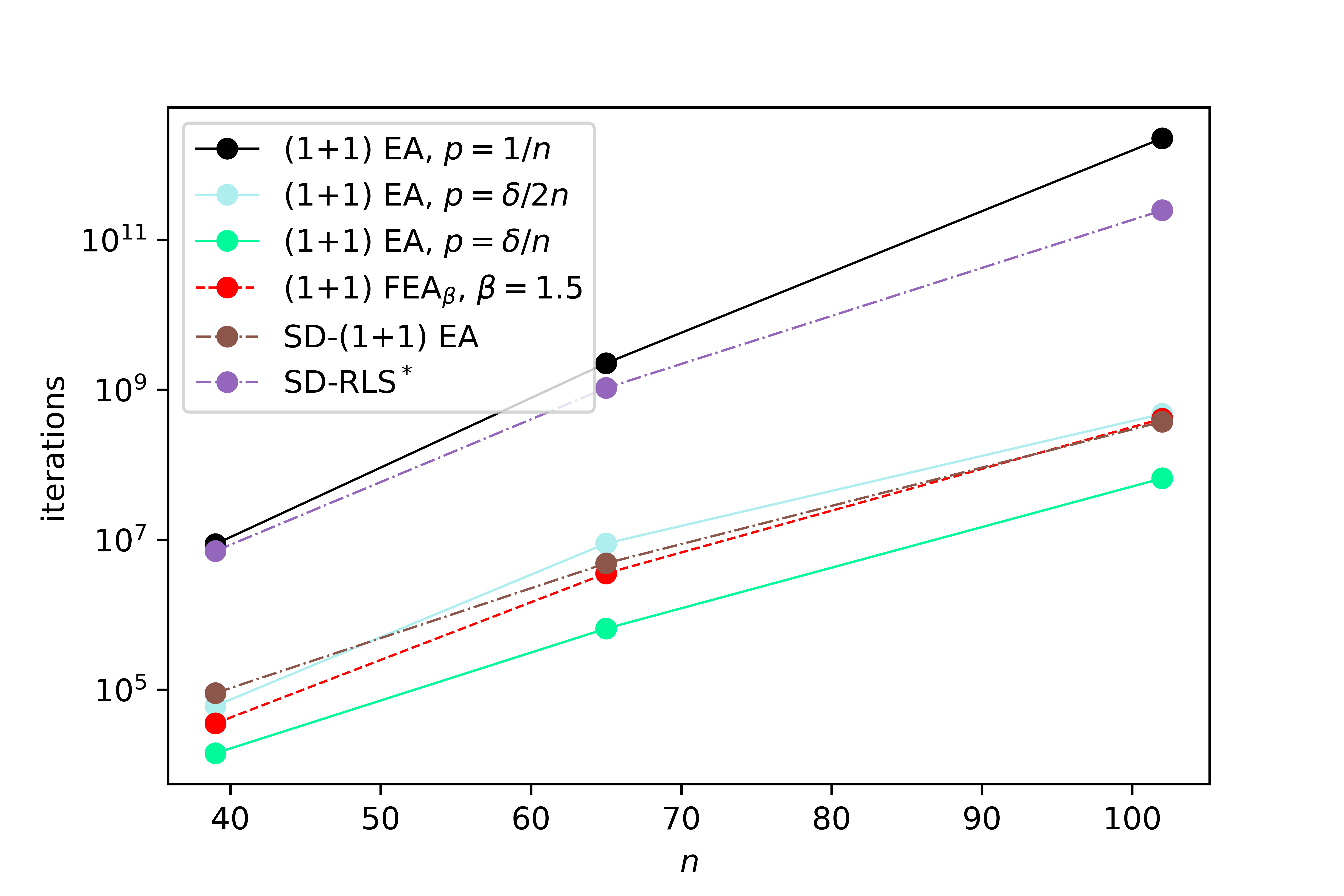}
\caption{Optimization times on $\jump_{k,\delta}$ with $k=4n^{0.3}$, $\delta=\frac k2$.}
\label{reg3}
\end{figure}
\begin{figure}[!htb] 
\centering
\includegraphics[width=10cm]{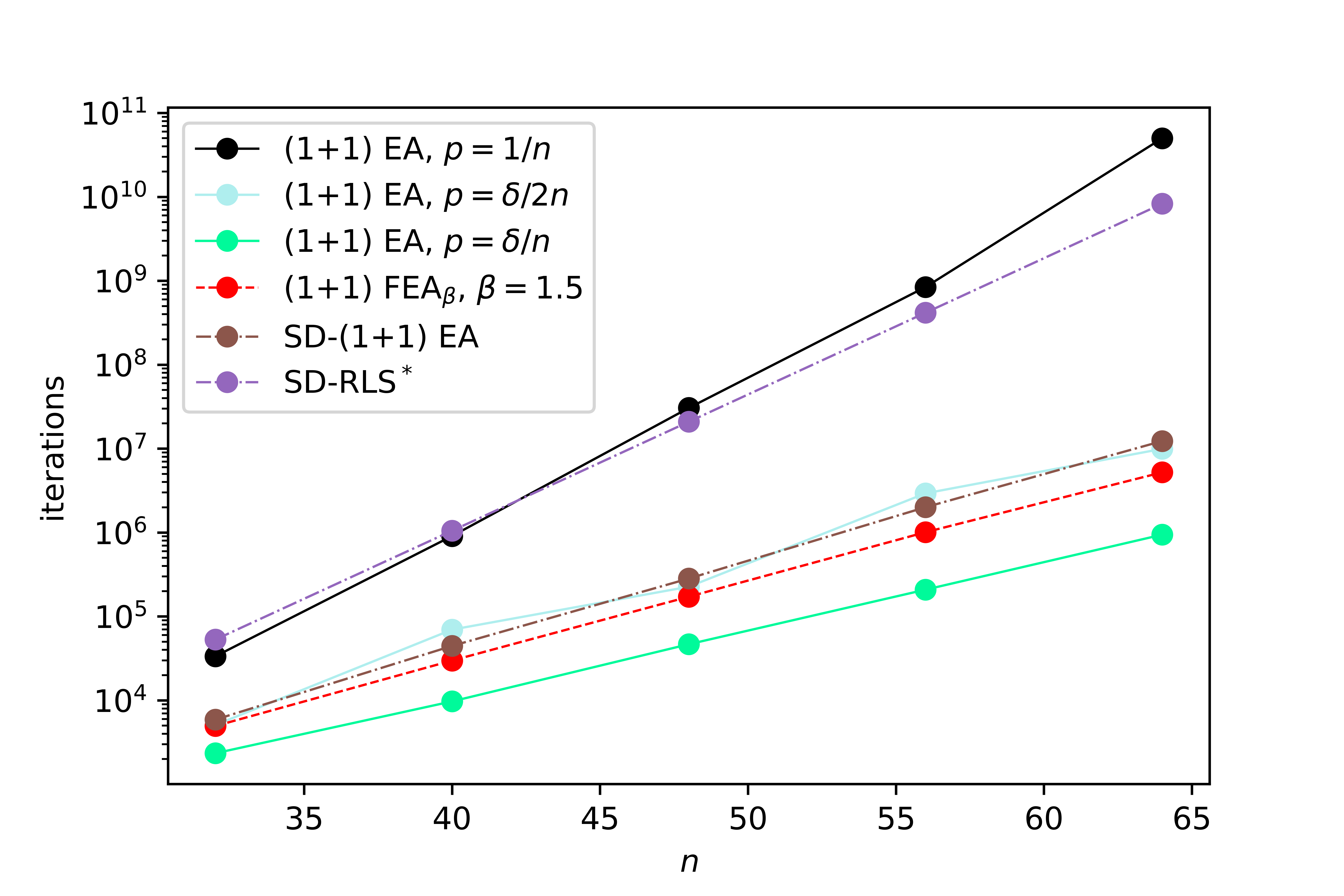}
\caption{Optimization times on $\jump_{k,\delta}$ with $k=\frac n4$, $\delta=\frac k2$.}
\label{reg4}
\end{figure}

Overall, these experimental results tend to confirm that our results, although asymptotic, are verified in simple instances of the problem as well.
The \oea with fixed mutation rate behaves as described by Theorem~\ref{optEA}. Out of the three mutation rates, $p=\frac 1n$ always gives the slowest runtime, $p=\frac{\delta}{n}$ has the fastest, and it is more efficient than the sub-optimal $p=\frac{\delta}{2n}$, by an exponential factor coherent with the theory. The SD-RLS$^*$ is the fastest on the classical $\jump_k$. It beats the \oea with optimal mutation rate $\delta/n$ by a factor of approximately $3$ (which is coherent with the theoretical factor, that is asymptotically $e$). However, Figure \ref{reg1} shows that it loses this advance as soon as $\delta<k$, even though it stays efficient. On the last three regimes, SD-RLS$^*$ is increasingly slow in comparison to the other algorithms, as it does not benefit from the speed-up induced by the increase of $\binom{k}{\delta}$. On the classical $\jump_k$, the SD-\oea is equivalent to the optimal \oea (which is coherent with the theoretical results of~\cite{RajabiW20}). It suffers on generalized jump functions, but not as consequently as the SD-RLS$^*$, as it stays of the same order as the \ofea. This is rather surprising: the change of mutation operator seems to induce a drastic change in behavior on $\jump_{k,\delta}$. We believe that studying the reasons for this phenomenon could be of great interest. Finally, the \ofea has a consistent behaviour. It is slower than the optimal \oea, but the runtime ratio remains stable throughout all the experiments, as expected theoretically.

\section{Conclusion}

In this work, we proposed a natural extension of the jump function benchmark class, which has a valley of low fitness in an arbitrary interval of Hamming distances from the global optimum. Our rigorous runtime analysis of different variants of the \oea on this function class showed that some previous results naturally extend, whereas others do not. The result that the fast \oea significantly outperforms the classic \oea on jump functions directly extends to our generalization (when \new{$k \le \frac{n^{1/3}}{\ln{n}}$}) as both runtimes simply improve by a factor of asymptotically $\binom{k}{\delta}$. The result that the \oea with stagnation detection and $k$-bit mutation (SD-RLS$^*$), the currently asymptotically fastest mutation-based algorithm for classic jump functions, beats the fast \oea by a moderate margin (a polynomial in~$k$), however, does not extend. Already in the standard regime, for any constant $K$ there are generalized jump functions such that the fast \oea beats the SD-RLS$^*$ by a factor of $\Omega(n^K)$. 

From this work, several open problems arise. Since for larger values of $k$ (and $\delta$ small) the \oea can (and, depending on the mutation rate, often will) jump over the valley not just to the first level above the valley, we could not determine precisely its asymptotic runtime outside the standard regime \new{$k \le \frac{n^{1/3}}{\ln{n}}$}. A better understanding of this regime would be highly desirable, among others, because here crossing a valley of low fitness is significantly easier, giving the problem a very different characteristic. Note that this difficulty cannot show up in the analysis of classic jump functions, simply because there is just a single solution above the valley of low fitness. 

In this first work on generalized jump functions, we regarded two topics of recent interest, the fast \oea and the random local search with stagnation detection. Jump functions have been very helpful also to understand other important topics, among them how crossover can be profitable or how probabilistic model building algorithms cope with local optima. Extending any such previous works to generalized jump functions, and with this confirming or questioning the insights made in these works, is clearly an interesting direction for future research.

\section*{Acknowledgments}
This work was supported by a public grant as part of the
Investissements d'avenir project, reference ANR-11-LABX-0056-LMH,
LabEx LMH.

\newcommand{\etalchar}[1]{$^{#1}$}

%
%
%

}
\end{document}